%% file: papers.tex
\documentclass[sigconf]{acmart}

\usepackage[a-1b]{pdfx}

\usepackage{graphicx}
\usepackage{balance}
\usepackage{comment}
\usepackage{amsfonts}
\usepackage{color,soul}
\usepackage{hyperref}
\usepackage{algorithmic}
\usepackage{algorithm}
\usepackage{subcaption}
\usepackage{xspace}
\usepackage{boxedminipage}
\usepackage{diagbox}
\usepackage{multirow}
\usepackage{paralist}
\usepackage{amsmath}
\usepackage{enumitem}
\usepackage{array}
\usepackage{mdframed}
\usepackage[normalem]{ulem}
\usepackage{tabularx}
\usepackage{pgf,tikz}
\usepackage{pgfplots}
\usepackage{mathrsfs}
\usepackage{diagbox}
\usepackage{footnote}
\usepackage{booktabs}
\usepackage{wrapfig}
\usepackage{url}
\usepackage{arydshln}

\usepackage{dashbox}

\usepackage{amsmath, amsfonts, amsthm}
\theoremstyle{definition}
\newtheorem{definition}{Definition}
\newtheorem{theorem}{Theorem}
\newtheorem{corollary}{Corollary}

\usetikzlibrary{arrows}
\usetikzlibrary{patterns} 
\makesavenoteenv{tabularx}
\makesavenoteenv{table}

\definecolor{skyblue}{rgb}{0.53, 0.81, 0.92}

\newcommand{\cmmnt}[1]{\ignorespaces}

\newcommand{\algname}{ARCUS}

\renewcommand{\algorithmiccomment}[1]{/*~#1~*/}
\newcolumntype{L}[1]{>{\raggedright\let\newline\\\arraybackslash\hspace{0pt}}m{#1}}
\newcolumntype{C}[1]{>{\centering\arraybackslash}p{#1}}

\pgfplotsset{compat=1.3}
\pgfplotsset{filter discard warning=false}

\AtBeginDocument{%
  \providecommand\BibTeX{{%
    \normalfont B\kern-0.5em{\scshape i\kern-0.25em b}\kern-0.8em\TeX}}}

\copyrightyear{2022}
\acmYear{2022}
\setcopyright{acmcopyright}
\acmConference[KDD '22]{Proceedings of the 28th ACM SIGKDD Conference on Knowledge Discovery and Data Mining}{August 14--18, 2022}{Washington, DC, USA}
\acmBooktitle{Proceedings of the 28th ACM SIGKDD Conference on Knowledge Discovery and Data Mining (KDD '22), August 14--18, 2022, Washington, DC, USA}
\acmPrice{15.00}
\acmISBN{978-1-4503-9385-0/22/08}
\acmDOI{10.1145/3534678.3539348}

\begin{document}

\title{Adaptive Model Pooling for Online Deep Anomaly Detection from a Complex Evolving Data Stream}

\author{Susik Yoon}
\affiliation{%
  \institution{UIUC}
   \country{Illinois, USA}
}
\email{susik@illinois.edu}

\author{Youngjun Lee}
\affiliation{%
  \institution{KAIST}
  \country{Daejeon, Korea}
}
\email{youngjun.lee@kaist.ac.kr}

\author{Jae-Gil Lee}
\authornote{Corresponding author.}
\affiliation{%
  \institution{KAIST}
  \country{Daejeon, Korea}
}
\email{jaegil@kaist.ac.kr}

\author{Byung Suk Lee}
\affiliation{%
  \institution{University of Vermont}
  \country{Vermont, USA}
}
\email{bslee@uvm.edu}
\renewcommand{\shortauthors}{Yoon et al.}

\begin{abstract}
Online anomaly detection from a data stream is critical for the safety and security of many applications but is facing severe challenges due to \emph{complex and evolving} data streams from IoT devices and cloud-based infrastructures. Unfortunately, existing approaches fall too short for these challenges; online anomaly detection methods bear the burden of handling the complexity while offline deep anomaly detection methods suffer from the evolving data distribution. This paper presents a framework for online deep anomaly detection, \algname{}, which can be instantiated with any autoencoder-based deep anomaly detection methods. It handles the complex and evolving data streams using an \emph{adaptive model pooling} approach with two novel techniques---\emph{concept-driven inference} and \emph{drift-aware model pool update}; the former detects anomalies with a combination of models most appropriate for the complexity, and the latter adapts the model pool dynamically to fit the evolving data streams. In comprehensive experiments with ten data sets which are both high-dimensional and concept-drifted, \algname{} improved the anomaly detection accuracy of the streaming variants of state-of-the-art autoencoder-based methods and that of the state-of-the-art streaming anomaly detection methods by up to 22\% and 37\%, respectively. 
\end{abstract}



\maketitle

\input{01-Introduction}
\input{02-RelatedWork.tex}
\input{03-Preliminaries.tex}
\input{04-Framework.tex}
\input{05-Experiments.tex}
\input{06-Conclusion}

\begin{acks}
This work was partly supported by Samsung Electronics Co., Ltd. (IO201211-08051-01) through the Strategic Collaboration Academic Program and Institute of Information \& Communications Technology Planning \& Evaluation\,(IITP) grant funded by the Korea government\,(MSIT) (No.\ 2022-0-00157, Robust, Fair, Extensible Data-Centric Continual Learning). The first author was also supported by Basic Science Research Program through the National Research Foundation of Korea\,(NRF) funded by the Ministry of Education (2021R1A6A3A14043765).
\end{acks}


\bibliographystyle{abbrv}
\bibliography{papers}


\input{07-Appendix.tex}

\end{document}

%% file: 01-Introduction.tex
\section{Introduction}
\label{sec:Introduction}
\subsection{Background and Motivation}
An anomaly can be identified as a data point that has different characteristics from a majority of other data points, and means a novel observation, a certain failure, unexpected noise, etc. in a system of interest. Anomaly detection has numerous real-world applications such as fraud detection in financial institutions and abnormality detection in healthcare devices\,\cite{DAD}

Today, a \emph{complex} data stream is common from IoT devices and cloud-based infrastructures, where data items with hundreds of features of often unknown correlations and heterogeneous data types are continuously arriving. This complexity brings about an insurmountable challenge to anomaly detection, thereby necessitating significant preprocessing or heavy model training to cope with the complexity. The challenge is aggravated when the data stream is \emph{evolving}, thereby making the preprocessing or the trained models outdated quickly. This phenomenon is often referred to as a \emph{concept drift}\,\cite{Lu18}, where properties of a target domain (i.e., concept) change arbitrarily. Such complex and evolving data streams are observed in various real-world situations\,\cite{RIALTO,GAS,INSECTS}; e.g., gas sensor value streams to monitor gas leaks with varying gas concentrations.


This paper concerns the computational method to deal with complex evolving data streams. \emph{Deep anomaly detection}\,\cite{DAD} based on a deep neural network has proven to handle the complexity effectively, better than classical methods (e.g., $k$ nearest neighbors\,\cite{Kno98})\,\cite{REPEN, Ruf20}. In particular, an autoencoder\,(AE) has been widely used, as it is appropriate for an \emph{unsupervised setting} that is natural for anomaly detection with rare labels. Existing state-of-the-art AE-based methods\,\cite{RAPP,RSRAE,DAGMM}, however, are designed for an offline setting and, thus, cannot effectively cope with evolving data streams. There are a few recurrent neural network\,(RNN)-based methods proposed for time series anomaly detection\,\cite{REBM, LSTM-ED, incrRNN1, incrRNN2}. However, they focus on learning temporal relationships inside local sequences and incrementally updating a single model, which are not suitable to handle arbitrarily evolving data streams. 

The goal of this paper, thus, is to provide a novel framework for \emph{online deep anomaly detection} that adopts the existing deep anomaly detection methods adaptively in an online setting, thereby effectively dealing with the complexity and evolution challenges of a data stream. 

\begin{figure*}[!t]
    \begin{subfigure}[b]{\columnwidth}
        \centering
        \includegraphics[width=\columnwidth]{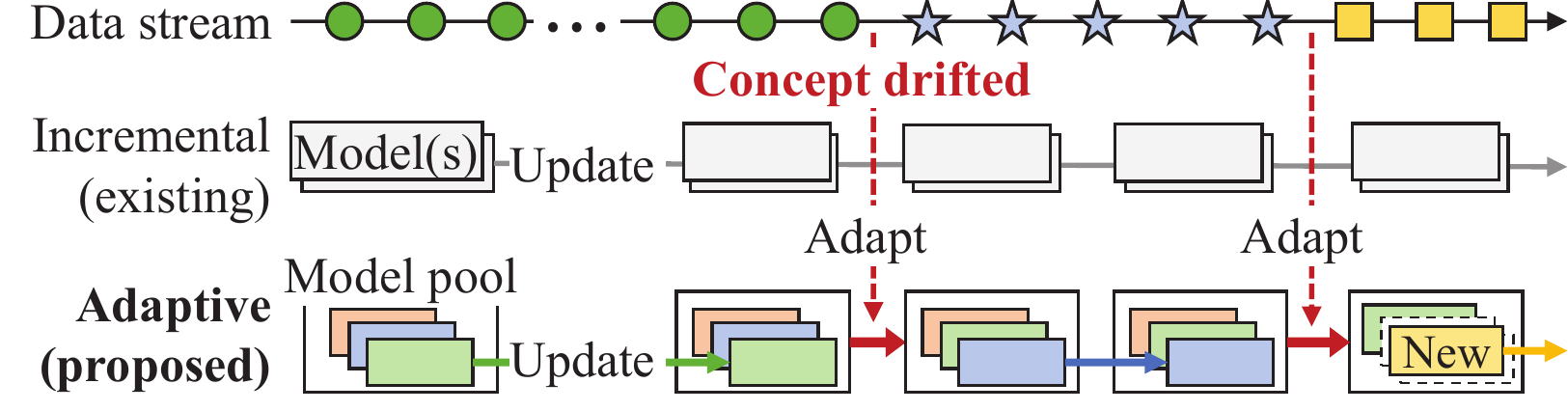}
        \caption{Processing flow of online deep anomaly detection. The marks on the timeline represent the concepts of data points.}
        \label{fig:main_idea}
    \end{subfigure}    
    \hfill
    \begin{subfigure}[b]{\columnwidth}
        \centering
        \includegraphics[width=\columnwidth]{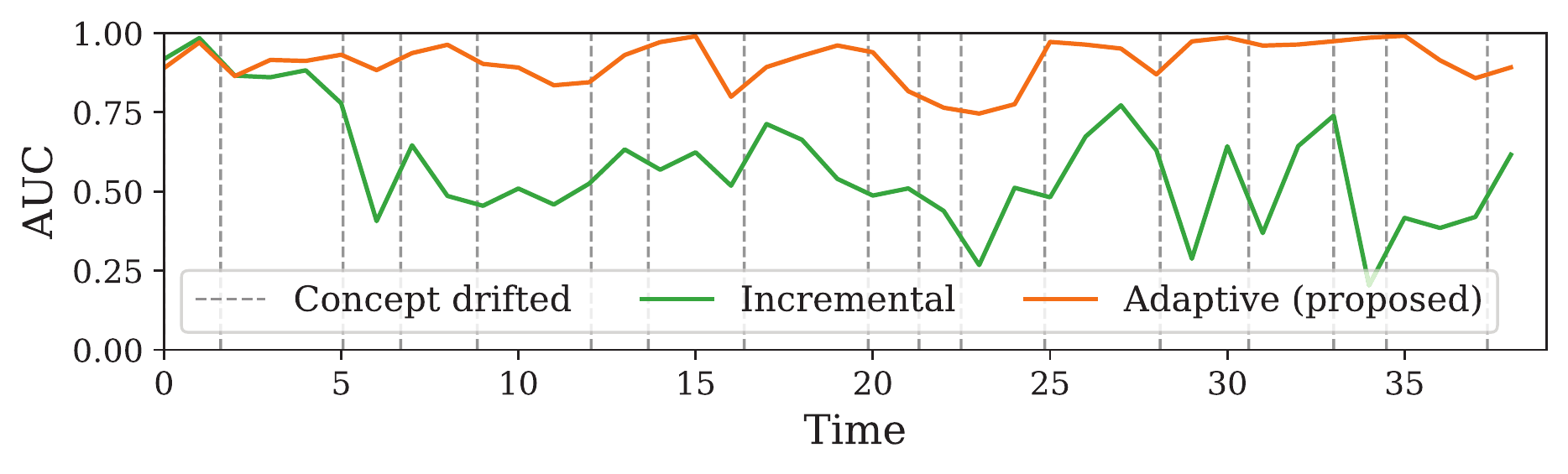}
        \vspace{-0.6cm}
        \caption{Detection accuracy in AUC (using an AE-based model for an MNIST data set simulated with abrupt and recurrent concept drifts).}
        \label{fig:main_idea_auc}
    \end{subfigure}    
    \vspace{-0.3cm}
    \caption{Comparison of incremental and adaptive approaches.}
    \vspace{-0.3cm}
\end{figure*}

\vspace*{-0.5em}
\subsection{Main Idea}
Undoubtedly, using a pre-trained fixed model or creating a new model repeatedly would not work at all to handle a complex evolving data stream; they are either too ineffective or too inefficient. A common approach of the existing streaming anomaly detection methods is to build an initial model and incrementally update the model over a data stream\,\cite{Yoon20, Yoon21, MSTREAM, Sud16} which can be easily applicable to deep anomaly detection methods. However, this \emph{incremental approach} adapts a model to only the latest data points, regardless of how data streams evolve. With arbitrary concept drifts in data streams, the incremental approach could be ineffective as it needs some time to perfectly adapt to new concepts and inefficient as it quickly forgets the previous concept that may reoccur in the future.

The main idea in this paper is to use \textbf{\emph{adaptive model pooling}}, which manages multiple assorted models in both inference over and adaptation to a complex evolving data stream. Faced with concept drifts involving the indeterminate number of multiple patterns, a fixed model or models cannot handle all of them. A model pooling approach thus allows multiple models to work together adaptively to handle multiple and time-varying concept drifts, thereby achieving versatile anomaly detection performance for a varying number of unexpected concept drifts. Unlike the existing ensemble approaches\,\cite{IF,AEE, Sud16} where the set of models is fixed in advance, the model pool membership is dynamically managed over time.

As illustrated in Figure \ref{fig:main_idea}, the incremental approach (top) updates a fixed model or models without regard to concept drifts, whereas the adaptive approach (bottom) uses a model pool and adjusts it in response to concept drifts---by using the best combination of existing models or creating a new model. As shown in Figure \ref{fig:main_idea_auc}, this adaptive model pooling brings a clear advantage in anomaly detection accuracy when arbitrary concept drifts occur.

This paper realizes the adaptive model pooling with \emph{autoencoder-based} deep anomaly detection models, which have shown state-of-the-art performances in unsupervised anomaly detection. Specifically, we propose a novel framework \textbf{\algname{}}\,(\ul{A}daptive framework fo\ul{R} online deep anomaly dete\ul{C}tion \ul{U}nder a complex evolving data \ul{S}tream), which employs two key techniques:

\begin{itemize}[leftmargin=10pt, noitemsep]
\vspace{-0.2cm}
\item \emph{\textul{Concept-driven inference}}: \algname{} calculates the anomaly scores of incoming data points using the best combination of models in the model pool. While optimizing individual models for different sets of data points, \algname{} estimates the reliability of each model against the given data points and decides how much each model contributes to the final anomaly score, to maximize the usability of the model pool in varying concepts.
\item \emph{\textul{Concept drift-aware update}}: \algname{} continuously monitors the reliability of the model pool for the incoming data points. When the model pool is assessed inadequate for the latest data points, presumably due to a concept drift, \algname{} updates the model pool to incorporate a new model optimized for the new concept of data points while keeping the model pool as compact as possible. This update enables \algname{} to efficiently keep the best performance regardless of the patterns of concept drifts.
\end{itemize}

\subsection{Highlights}\label{sec:highlights}
\begin{itemize}[leftmargin=10pt, noitemsep]
\item 
To the best of our knowledge, this is the first work that proposes an adaptive model pooling mechanism of deep anomaly detection models which addresses both the complexity and concept drift challenges of a data stream.
\item 
With the model pooling in place, this paper proposes a novel framework \textbf{\algname{}} equipped with concept-driven inference and concept drift-aware update. \algname{} can be implemented with any existing AE-based anomaly detection model. For reproducibility, the source code of \algname{} is publicly available\footnote{\url{https://github.com/kaist-dmlab/ARCUS}}.
\item 
Comprehensive experiments are conducted using ten data sets which are both high-dimensional and concept-drifted. \algname{}, when implemented using three state-of-the-art AE-based models, achieved up to $22\%$ higher anomaly detection accuracy than their streaming variants and surpassed the best accuracy results of the existing online anomaly detection algorithms by up to $37\%$.
\vspace{-0.1cm}
\end{itemize}


%% file: 02-RelatedWork.tex
\section{Related work}
\label{sec:RelatedWork}

\subsection{Deep Anomaly Detection}
The recent rapid advancement of a deep neural network has led to many deep anomaly detection methods with various types of approaches (e.g., AE, RNN, or generative adversarial networks (GAN))\,\cite{DAD}. Among them, the autoencoder has been widely studied and achieved the state-of-the-art performances\,\cite{DAGMM, RSRAE, RAPP}, thanks to its unsupervised but effective capability to remove noisy or anomalous information in the input. DAGMM\,\cite{DAGMM} combines an AE with the Gaussian mixture model to detect anomalies by predicting the Gaussian mixture membership of a data instance in a low-dimensional representation obtained from an AE. RSRAE\,\cite{RSRAE} added a linear transformation layer after a latent space of an AE to learn the hidden linear structure of the non-linearly embedded data points by an encoder. RAPP\,\cite{RAPP} investigates hidden activation values of layers in an AE in the same way as calculating reconstruction errors to verify the information loss during encoding and decoding processes. Although these methods have shown high anomaly detection accuracy for static data sets, none of them is geared for data streams. Besides, it is worth mentioning that the AE-based ensemble method\,\cite{AEE} was proposed to use multiple AE models differentiated by random edge sampling and trained with adaptive data sampling for anomaly detection. In addition to that it is still designed for an offline setting, our AE-based adaptive model pooling is fundamentally different from the AE-based ensemble in that we manage the dynamically changing number of models which are selectively adapted to evolving data distributions.

\smallskip
\subsection{Streaming Anomaly Detection}
Popular anomaly detection methods based on different approaches, such as $k$ nearest neighbors (kNN), kernel density estimation (KDE), isolation forest (IF)\,\cite{IF}, and locality sensitive hashing (LSH), have been actively extended to work online on a data stream. The kNN-based streaming methods detect outliers based on queries (e.g., NETS\,\cite{Yoon19} with set-based update and MDUAL\,\cite{Yoon21} with data-query duality) or local outlier factor (e.g., MiLOF\,\cite{Sal16} with summarization and DILOF\,\cite{Na18} with sampling). The KDE-based method STARE\,\cite{Yoon20} employs stationary region skipping for updating densities and anomaly scores. The LSH-based method MStream\,\cite{MSTREAM} uses two hashing-based features accompanied with dimensionality reduction techniques including the AE. The IF-based method RRCF\,\cite{Sud16} manages an ensemble of decision trees with a sketching technique. While they are effective in handling evolving data streams to some extent by adopting window-based processing, their focus is on reducing the computational overhead of incremental updates of a model or a pre-fixed ensemble of models for the changing data distributions. Further, they fundamentally rely on manual feature engineering such as dimensionality reduction, random sub-sampling, and linear feature transformation to deal with complex data, which often leads to sub-optimal results and limited scalability\,\cite{REPEN, Ruf20}. In this work, \algname{} proactively adapts to evolving data streams with a dynamic model pool while being free of ad-hoc feature engineering thanks to an AE-based deep anomaly detection model.

\smallskip
\subsection{Online Deep Learning}
Online deep learning from a data stream is to manage a deep neural network-based model over a data stream for continuous inference and update. Model adaptation\,\cite{HBP, IADM} and incremental update\,\cite{incrRNN1,incrRNN2} are approaches popularly used in existing studies. In the model adaptation approach, the hedge backpropagation is used to determine the weight of each hidden layer when combining the outputs from all layers\,\cite{HBP}, and an attention model and a Fisher matrix are employed to adjust layer weights and exploit the previously trained weights\,\cite{IADM}. In the incremental update approach, adaptive gradient learning adjusts the weights of gradients in response to data distribution changes\,\cite{incrRNN1}, and local normalization handles statistical shifts in model outputs\,\cite{incrRNN2}. While these studies pioneered online deep learning, they focus on adapting a \emph{single} model to a data stream. Moreover, they are designed for supervised classification\,\cite{HBP, IADM} and prediction-based time series anomaly detection\,\cite{incrRNN1,incrRNN2}, which are outside the scope of this work.

%% file: 03-Preliminaries.tex
\smallskip
\section{Preliminaries}
\label{sec:Preliminaries}

\subsection{Problem Setting}\label{sec:problem_setting}
\begin{figure}[!t]
    \centering
    \includegraphics[width=\columnwidth]{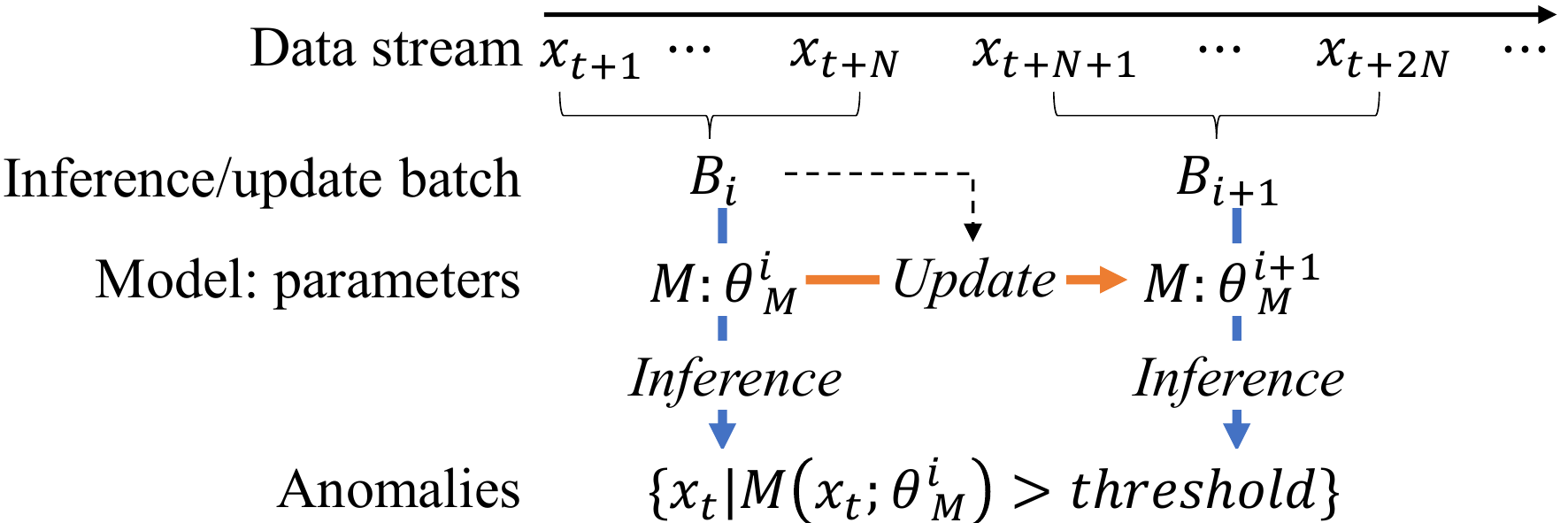}
    \vspace{-0.5cm}
    \caption{Batch-based continuous anomaly detection.}
    \vspace{-0.0cm}
    \label{fig:problem_setting}
\end{figure}

Given an unbounded sequence of data points $\langle\ldots$, $x_{t-1}$, $x_t$, $x_{t+1}$, $\ldots\rangle$ arriving in a data stream, an anomaly detection model $M$ with parameters $\theta_M$ calculates anomaly scores of the individual data points, $\langle\ldots$, $M(x_{t-1};\!\theta_M), M(x_t;\!\theta_M), M(x_{t+1};\!\theta_M)$, $\ldots\rangle$, continuously while updating the parameters $\theta_M$ unsupervised and reports the data points whose scores exceed a threshold as anomalies. Figure \ref{fig:problem_setting} illustrates this anomaly detection performed in streaming batch processing. A batch $B$ of data points newly coming from a data stream is used by $M$ for inference first and then used to update the parameters $\theta_M$ afterward, following the \emph{prequential evaluation} scheme\,\cite{Gam13} designed to evaluate an online stream learning algorithm by interleaving training and testing within the same batch. The inference returns a set of anomaly scores of the data points in the batch $B$, i.e., $\{ M(x_{i};\!\theta_M) \,|\, x_i \in B \}$, abbreviated as $M(B;\!\theta_M)$.

\subsection{Autoencoder-based Anomaly Detection}
\label{sec:autoencoder}
An autoencoder (AE) is a feed-forward neural network with a \emph{encoder} $E$ and a \emph{decoder} $D$, aiming at reconstructing an input as exactly as possible, i.e., $\text{min}_{E,D} \|X-D(Z)\|^2$ where $Z = E(X)$ is the latent representation of an input $X$. Typically, the reconstruction error for a given input $X$ is used as the anomaly score of $X$. Motivated by the latest AE-based anomaly detection models\,\cite{RDA, DAGMM, RSRAE, RAPP}, which have shown the state-of-the-art performances, we use an AE and its variants as a base anomaly detection model $M$ in this work. 

\subsection{Concept Drift}
As a concept refers to a certain distributional or statistical property of data in a domain, a \emph{concept drift} refers to an extrinsic phenomenon of the concept changing arbitrarily over time\,\cite{Lu18}. Formally, a concept drift occurs at time $t$ if the joint probability of input data points $X$ and their label $y$ changes at time $t$, that is, $P_t(X,y) \neq P_{t+1}(X,y)$. Since $P_t(X,y) = P_t(X)P_t(y|X)$, the source of such a concept drift can be identified as one of the following three: (i) $P_t(X) \neq P_{t+1}(X)$, i.e., change in the data distribution; (ii) $P_t(y|X) \neq P_{t+1}(y|X)$, i.e., change in the anomaly decision boundary; and (iii) both (i) and (ii). Once a concept drift occurs, the current anomaly detection model becomes obsolete and should be updated to learn the new concept. Since the drift can occur in various forms, such as ``sudden,'' ``gradual,'' ``incremental,'' and ``reoccurring''\,\cite{Lu18, INSECTS}, it is important to have a versatile mechanism to handle all these forms equally well. 

%% file: 04-Framework.tex
\section{The \algname{} Framework}
\label{sec:Framework}

\subsection{Overview}

\begin{figure}[!t]
    \centering
    \includegraphics[width=\columnwidth]{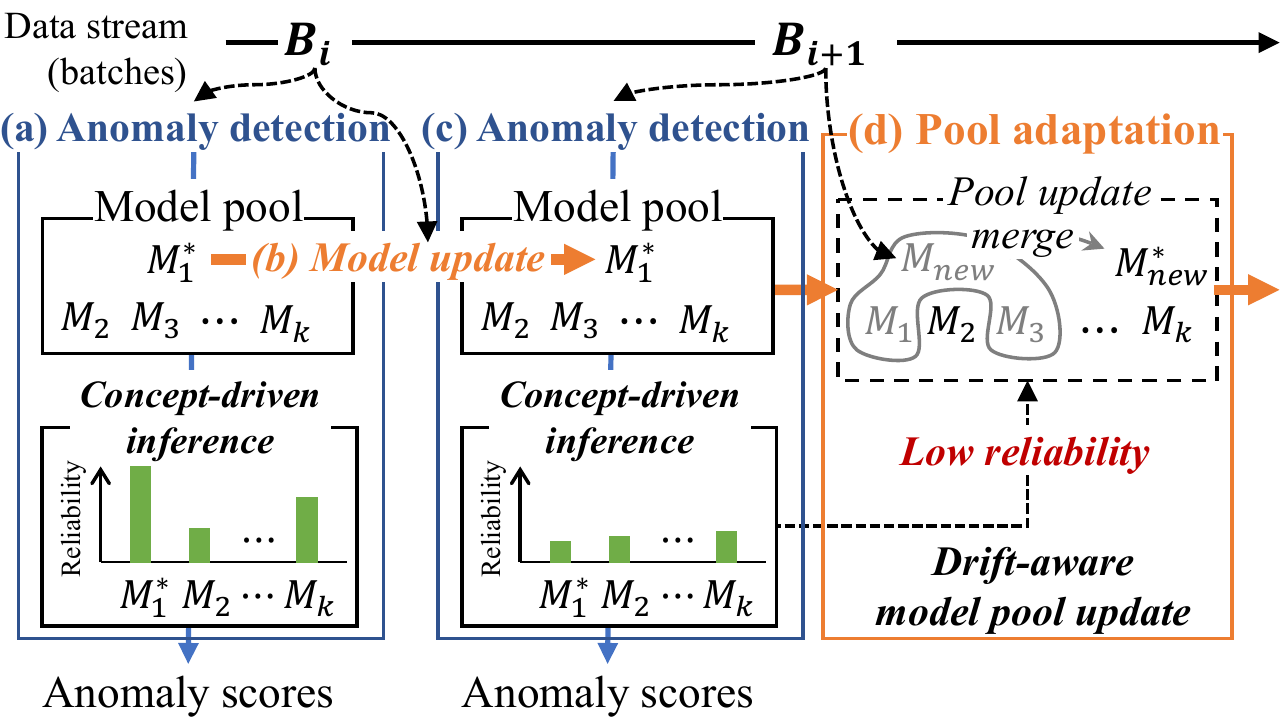}
    \vspace{-0.7cm}
    \caption{Overview of the procedure of \algname{}.}
    \vspace{-0.3cm}
    \label{fig:overall_framework}
\end{figure}

\begin{algorithm}[t!]
\begin{algorithmic}[1]
\small
\caption{\bf Overall Procedure of \algname{}}
\label{alg:overall}
\REQUIRE a data stream $DS$, an AE-based model $M$, a reliability threshold $\alpha$, a similarity threshold $\gamma$
\ENSURE anomaly scores $S$ of data points for each batch
\STATE Initialize a model pool $P$ with a model built from the first batch of $DS$;
\FOR{\textbf{each} batch $B$ of data points from $DS$}
  \STATE \algorithmiccomment {{\color{blue} {\sc 1. Anomaly Detection}}}
    \STATE Calculate the anomaly scores $S$ of data points in $B$ by $P$;
    \label{alg:detection}
    \STATE \algorithmiccomment {{\color{blue} {\sc 2. Model Pool Adaptation}}}
    \STATE Estimate the reliability $R_P$ of the model pool $P$; \label{alg:reliability}
    \IF{$R_P \geq \alpha$} \label{alg:model_update_begin}
        \STATE {{\color{blue}// Model update (i.e., minor update)}}
        \STATE Select the most reliable model $M^*$ in the model pool $P$;
        \STATE Perform incremental update on $M^*$ with $B$;
        \label{alg:model_update_end}
    \ELSE \label{alg:pool_update_begin}
        \STATE {{\color{blue}// Pool update (i.e., major update)}}
        \STATE Initialize a new model $M_{new}$ with $B$;
        \STATE Compact $P$ to $P^*$ by recursively merging $M_{new}$ with $M \in P$ whose similarity to $M_{new}$ exceeds $\gamma$;
    \ENDIF  \label{alg:pool_update_end}
    \RETURN anomaly scores $S$;\label{alg:return_anomaly_scores}  
\ENDFOR
\end{algorithmic}
\end{algorithm}
\newlength{\oldtextfloatsep}\setlength{\oldtextfloatsep}{\textfloatsep}
\setlength{\textfloatsep}{5pt}

\noindent
\algname{} is an online anomaly detection framework designed for any AE-based deep anomaly detection model. \algname{} manages a pool of models to perform inference over a batch of data stream, and then updates the model pool to adapt to new concepts detected in the batch. The overall procedure of \algname{} is illustrated in Figure \ref{fig:overall_framework} and outlined in Algorithm \ref{alg:overall}. Once a model pool is initialized with a model created for the first batch, \algname{} repeats, for every batch, anomaly detection using \emph{concept-driven inference} and model pool adaptation using \emph{concept drift-aware update}. The anomaly detection step calculates the anomaly scores of the data points in the current batch based on the reliability of individual models in the pool against the batch (Line \ref{alg:detection} and Figures \ref{fig:overall_framework}a and \ref{fig:overall_framework}c). The model pool adaptation step evaluates the overall reliability of the model pool against the current batch (Line \ref{alg:reliability}) and updates the model pool as needed (Lines \ref{alg:model_update_begin}$-$\ref{alg:pool_update_end}). Specifically, if the model pool fits well, \algname{} keeps the current model pool and updates only the model contributing most to the pool's reliability (Lines \ref{alg:model_update_begin}$-$\ref{alg:model_update_end} and Figure \ref{fig:overall_framework}b); otherwise, \algname{} creates a new model and then merges it with similar existing models (to keep the model pool as compact as possible) (Lines \ref{alg:pool_update_begin}$-$\ref{alg:pool_update_end} and Figure \ref{fig:overall_framework}d). \algname{} then returns the anomaly scores of the current batch (Line \ref{alg:return_anomaly_scores}). These two steps are discussed in detail in the following sections.


\smallskip
\subsection{Model Pooling}
\label{sec:model_and_model_pool}

A \emph{model} and a \emph{model pool} used in \algname{} are formalized as follows.

\begin{definition}
\label{def:model} 
({\sc Model})
A \emph{model} $M$ is represented by an AE with an encoder $E_M$, a decoder $D_M$, and any additional components needed by the specific AE. \hfill $\square$
\end{definition}


\begin{definition}
\label{def:model_pool} 
({\sc Model Pool})
A \emph{model pool} $P$ is a set of models $\{M_1, M_2, \ldots, M_k\}$ that share the same architecture but have different parameters $\theta_{M_i}$ learned from different input batches.
\hfill $\square$
\end{definition}

The membership of a model pool needs to find a balance between the efficacy and efficiency of anomaly detection. In one extreme, the pool may contain one dedicated model for every observed concept to achieve the highest accuracy while incurring significant overhead to keep all models. In the other extreme, the pool may contain only one model to minimize the update cost while sacrificing accuracy. \algname{} finds the necessary balance as adapting to data streams so that the accuracy considering all models in the pool is maximized while keeping the pool as compact as possible.

\subsection{Anomaly Detection}
\label{sec:anomaly_detection}

\subsubsection{\textbf{Model Reliability}}
The concept reflected in a given model may be different from the concept of the current batch, especially when a concept drift has occurred. \algname{} estimates the \emph{reliability} of a model by comparing the concept learned by the model with that of the current batch and uses it to calculate the concept-driven anomaly score.
A straightforward way to estimate the model reliability is to directly investigate a sequence of error rates reported by the model, as adopted in the existing studies\,\cite{Bif07, Fri14}. However, they deal with a supervised setting where the error rate of a model is reported immediately and, thus, become impractical when the true labels of anomalies are not available in this unsupervised online anomaly detection. Thus, we exploit a set $M(B;\theta_M)$ of anomaly scores returned by a model $M$ on a batch $B$, which can be obtained immediately during inference without additional work while reflecting the normal characteristics of the data points in the batch.

Let $M(B_{Curr};\theta_M)$ be the sets of anomaly scores on the current batch and $M(B_{Last};\theta_M)$ be those on the last batch used to update the model. Then, given the model $M$, the statistical significance of the difference between $M(B_{Curr};\theta_M)$ and $M(B_{Last};\theta_M)$ indicates how reliable the model could be to the current batch. To quantify such significance, we adopt the Hoeffding's Inequality\,\cite{Hoe94}-based mean difference bound (see Theorem \ref{theorem:hoeffding_bound}), which has been widely used for detecting statistical changes in streaming values\,\cite{Bif07, Fri14, Lu18}.

\begin{theorem}
\label{theorem:hoeffding_bound} 
({\sc Hoeffding's Inequality-based Mean Difference Bound}) \cite{Fri14}
Given independent random variables $X$ and $Y$ bounded by $[a_{min},a_{max}]$, the probability of the sample mean difference between $\overline{X}=\frac{1}{n}\sum_{i=1}^n X_i$ and $\overline{Y}=\frac{1}{m}\sum_{j=1}^m Y_j$ is bounded by
\begin{equation}
\small
\label{eq:hoeffding}
    Pr\{|\overline{X} - \overline{Y}| \geq \epsilon\} \leq e^{\frac{-2\epsilon^2}{ (n^{-1}+m^{-1})(a_{max}-a_{min})^2}}. \qed
\end{equation}
\end{theorem}

Applying Theorem \ref{theorem:hoeffding_bound} to $M(B_{Curr};\theta_M)$ and $M(B_{Last};\theta_M)$ gives the statistical significance of their difference in Corollary \ref{cor:model_reliability}.

\begin{corollary}
\label{cor:model_reliability} 
({\sc Concept Difference Bound}) Let $X$ and $Y$ be the independent anomaly scores returned by a model $M$, and let $M(B_{Curr};\!\theta_M)$ and  $M(B_{Last};\!\theta_M)$ be the sets of anomaly scores sampled from $X$ and $Y$, respectively. Then,
\begin{equation}
\label{eq:model_reliability}
\small
\begin{split}
    Pr\{|\overline{X} - \overline{Y}| \geq \epsilon \} &\leq e^{\frac{-2\epsilon^2}{(b^{-1}+b^{-1})(s_{max}-s_{min})^2}} =e^{\frac{-b\epsilon^2}{(s_{max}-s_{min})^2}},
\end{split}
\end{equation}
where $b$ is the batch size (i.e., $b=|B_{Curr}|=|B_{Last}|$),
\begin{equation}
\label{eq:model_reliability_aux}
\small
\begin{split}
\epsilon &= |avg(M(B_{Curr};\!\theta_M)) - avg(M(B_{Last};\!\theta_M))| \\
s_{max} &= max(max(M(B_{Curr};\!\theta_M)), max(M(B_{Last};\!\theta_M))),\text{ and}\\
s_{min} &= min(min(M(B_{Curr};\!\theta_M)), min(M(B_{Last};\!\theta_M))).
\end{split}
\end{equation}
\end{corollary}
\begin{proof}
Straightforward from Theorem \ref{theorem:hoeffding_bound}\,\cite{Fri14}.
\end{proof}

By using the probability bound (Eq.\ \eqref{eq:model_reliability}) in Corollary \ref{cor:model_reliability}, the \emph{reliability} of a single model is derived as in Definition \ref{def:model_reliability}. Thus, the probability that the sample mean difference of the anomaly scores by $M$ is higher than or equal to $\epsilon$ is at most $r_M$.

\begin{definition}
\label{def:model_reliability}
({\sc Model Reliability}) 
The \emph{model reliability} $r_M$ of $M$ for the current batch $B_{Curr}$ is defined as
\begin{equation}
r_M = e^{\frac{-b\epsilon^2}{(s_{max}-s_{min})^2}},
\end{equation}
where $\epsilon$, $s_{max}$, and $s_{min}$ are the same as those in Eq.\ \eqref{eq:model_reliability_aux}. \hfill $\Box$
\end{definition}
 
Note that deriving the model reliability is efficient as it requires only the lightweight anomaly score statistics (i.e., the triplet of $\langle min(M(B;\theta_M)),$ $max(M(B;\theta_M)),$ $avg(M(B;\theta_M)) \rangle$).

\smallskip
\subsubsection{\textbf{Concept-Driven Inference}}
The final anomaly score of a batch is determined by weighting the standardized anomaly scores of each model with its reliability (see Definition \ref{def:final_anomaly_score}), making the contribution of each model proportional to its reliability. 

\begin{definition}
\label{def:final_anomaly_score} 
({\sc Concept-driven Anomaly Score}) Given a set of models $\{M_1, \ldots, M_k\}$ in a model pool $P$ and the corresponding reliabilities $\{ r_{M_1}, \ldots, r_{M_k}\}$, the \emph{concept-driven anomaly score} $C_P(x)$ of a data point $x$ in the current batch $B$ is calculated as
\begin{equation}
\small
    C_P(x) = \sum_{i=1}^k r_{M_i}(\frac{M_i(x;\!\theta_{M_i}) - avg(M_i(B;\!\theta_{M_i}))}{std(M_i(B;\!\theta_{M_i}))}). \qed
\end{equation}
\end{definition}

\smallskip
\subsection{Model Pool Adaptation}
\label{sec:concept_drfit_adaptation}

\subsubsection{\textbf{Reliability of Model Pool}}
When a batch with a new concept that has never been seen arrives, none of the models in a model pool $P$  would do inference on the batch correctly. Thus, \algname{} estimates the overall reliability of a model pool to decide whether the model pool needs to be updated. 

By Definition \ref{def:model_reliability}, the probability that a model is not reliable for the current batch is $1-r_M$. At the same time, the models in $P$ are independent of one another since they have been separately updated with non-overlapping batches.
Then, the reliability of $P$ is defined by $1$ minus the probability that none of the models in $P$ are reliable, as formulated in Definition \ref{def:reliability_model_pool}.

\begin{definition}
\label{def:reliability_model_pool}
({\sc Model Pool Reliability}) Given a model pool $P = \{M_1,\ldots,M_k\}$, the \emph{reliability} $R_P$ of $P$ is $1-\prod_{i=1}^k(1-r_{M_i})$. \hfill $\square$
\end{definition}

\subsubsection{\textbf{Model Merging}}
The latent representation $Z$ learned by an AE-based model is expected to contain minimal but sufficient information of the input. If two models show similar latent representations of the same input, they must have been updated by the temporally separate disjoint batches of similar concepts, so merging them helps to remove redundancy in the model pool and also avoid overfitting. To this end, \algname{} uses the \emph{centered kernel alignment\,(CKA)} for measuring the similarity of two models (see Definition \ref{def:cka}), as it is known as the most appropriate similarity index for neural network representations since it is invariant to orthogonal transformation and isotropic scaling but not invariant to invertible linear transformation\,\cite{CKA}.

\input{Figures/Datasets.tex}

\begin{definition}
\label{def:cka} 
({\sc Model Similarity}) Given an input $X$ and two models $M_1$ and $M_2$, let $Z_1$ and $Z_2$ be the latent representations of $X$ by $M_1$ and $M_2$, respectively. For a kernel $\mathcal{K}$, let $K_{ij}^{Z}$ be $\mathcal{K}(z_i,z_j)$ for $z_i, z_j \in Z$. Then, the \emph{similarity} between $Z_1$ and $Z_2$, $Sim(Z_1,Z_2)$, is calculated as
\begin{equation}
\small
\label{eq:CKA}
    CKA(K^{Z_1},K^{Z_2}) =  \frac{HSIC(K^{Z_1},K^{Z_2})}{\sqrt{HSIC(K^{Z_1},K^{Z_1})HSIC(K^{Z_2},K^{Z_2})}},
\end{equation}
where $HSIC(K^{Z_1},K^{Z_2})$ is Hilbert-Schmidt Independence Criterion \cite{HSIC}. We used a linear kernel $\mathcal{K}$, which is simple but comparable with other kernels\,\cite{CKA}. Then, Eq.~\eqref{eq:CKA} becomes equivalent to
\begin{align*}
\small
\|{Z_1}^TZ_2\|^2_F/(\|{Z_1}^TZ_1\|_F\|{Z_2}^TZ_2\|_F),
\end{align*}
where $\|\cdot\|_F$ denotes the Frobenius norm.\hfill $\square$
\end{definition}

The merging of two models $M_1$ and $M_2$ is conducted by weighted averaging of their parameters (see Definition \ref{def:model_merge}). In federated learning, when local models with the same initialization of parameters are optimized via stochastic gradient descent, parameter averaging has been proven to be equivalent to gradient averaging and converging to the global model\,\cite{Li19, Mcm17}. Thus, an AE-based anomaly detection model merged from two models trained with two temporally separate disjoint batches of the same concept is guaranteed to \emph{converge} to a model trained with the entire data set of the concept.


\smallskip
\begin{definition}
\label{def:model_merge} 
({\sc Model Merging}) Given two models $M_1$ and $M_2$, their number of batches $N_{M_1}$ and $N_{M_2}$ used to update the models, and their parameters $\theta_{M_1}$ and $\theta_{M_2}$, the \emph{merged model} is defined to have the parameters 
\begin{equation}
\small
    \theta_{M_{merged}} = (N_{M_1}\theta_{M_1} + N_{M_2}\theta_{M_2})/(N_{M_1}+N_{M_2}). \qed
\end{equation}
\end{definition}

\smallskip
\subsubsection{\textbf{Drift-Aware Model Pool Update}}
\label{sec:drift_aware_pool_update}
\algname{} monitors the reliability of a model pool and triggers the update of the pool with a significance level $1-\alpha$. A model pool will be kept unchanged when it has at least a single highly reliable model\,(i.e., $r_M > \alpha$), but it will be adjusted when the models in the pool have only neutral reliability values\,(i.e., $r_M \ll \alpha$). We set the default value of $\alpha$ to 0.95, which is commonly used in the statistical significance test, meaning that only $5\%$ of the possibility that all models are not reliable will be allowed, and also confirmed to be valid in the sensitivity analysis in Section \ref{sec:sensitivity}.

Once the update of the model pool is triggered, \algname{} first creates a new model with the current batch and derives a compact model pool (see Definition \ref{def:compact_model_pool}) in a greedy way, by recursively merging the new model with the most similar model exceeding a similarity threshold $\gamma$. 

\begin{definition}
\label{def:compact_model_pool} 
({\sc Compact Model Pool}) Given a model pool $P$, a new model $M_{new}$ with a current batch $B$, and a similarity threshold $\gamma$, a \emph{compact model pool} $P^*$ satisfies
\begin{equation}
\small
\label{eq:compact_model_pool}
    \begin{gathered}
    \min_{P^*} |P^*|\\
    ~\text{s.t.}~ Sim(Z_{M_{new}}^B, Z_{M_i}^B) < \gamma ~\text{for every~} M_i \in P^* ~\text{and}~M_i \neq M_{new}, \\
    \text{where}~~ Z_{M_{new}}^B = E_{M_{new}}(B)~~\text{and}~~Z_{M_i}^B = E_{M_{i}}(B). \qed
    \end{gathered}
\end{equation}
\end{definition}

We set the default value of $\gamma$ in $[0,1]$ to 0.8 to allow both the diversity\,(i.e., $\gamma>0.5$) and the compactness\,(i.e., $\gamma<1$) of a model pool. We have confirmed that the similarity of models trained under the same concept is usually higher than 0.8 (in Figure \ref{fig:cdf_similarity} in Appendix \ref{apdx:algorithms}), which was also valid in the sensitivity analysis in Section \ref{sec:sensitivity}.

\subsection{Complexity Analysis}
\label{sec:complexity_analysis}

\begin{theorem}
\label{theorem:time_complexity} 
Given the batch size $b$, the number $e$ of epochs, the number $k$ of models in a model pool, and the model parameter size $T$, the time complexity of \algname{} is $O(b^2 + bT)$ and the space complexity of \algname{} is $O(kT)$.
\end{theorem}
\vspace*{-0.3cm}  
\begin{proof}
The time complexity for the anomaly detection step is $O(kbT)$ and that for the concept drift adaptation step is $O(ebT + kb^2 + kT)$ where is $O(ebT)$ for updating a model, $O(kb^2)$ is for calculating CKA-based similarities, and $O(kT)$ is for merging models. Since typically $b,T \gg e, k$, the total time complexity is $O(b^2 + bT)$. The space complexity for managing models is $O(kT)$ and that for managing anomaly score statistics is $O(k)$, and thus the total space complexity is $O(kT)$.
\end{proof}

%% file: Figures/Datasets.tex
\bgroup
\def\arraystretch{0.9}%
\begin{table*}[t!]
\centering
\caption{High-dimensional and concept-drifting data streams used for evaluation.}
\vspace{-0.3cm}
\small
\label{tbl:datasets}
\begin{tabular}{cccccc} \toprule
Data set & Description & \# Obj & \# Dim & Concept drift type & Anomaly target \\ \midrule
MNIST-AbrRec & Handwritten digits & 20,480 & 784 & Abrupt and recurrent & Arbitrary digits  \\ 
MNIST-GrdRec & Handwritten digits & 20,480 & 784 & Gradual and recurrent & Arbitrary digits  \\ 
FMNIST-AbrRec & Fashion items & 20,480 & 784 & Abrupt and recurrent & Arbitrary fashion items  \\ 
FMNIST-GrdRec & Fashion items & 20,480 & 784 & Gradual and recurrent & Arbitrary fashion items  \\ 
\hdashline
INSECTS-Abr & Optical sensor values for insects & 52,848 & 33 & Abrupt & The southern house mosquito  \\ 
INSECTS-Inc & Optical sensor values for insects & 57,018 & 33 & Incremental & The southern house mosquito  \\ 
INSECTS-IncGrd & Optical sensor values for insects & 24,150 & 33 & Incremental and gradual & The southern house mosquito  \\ 
INSECTS-IncRec & Optical sensor values for insects & 79,986 & 33 & Incremental and recurrent & The southern house mosquito  \\ 
GAS & Gas sensor values & 13,910 & 128 & Unknown & Acetaldehyde  \\ 
RIALTO & Building images around a bridge & 82,250 & 27 & Unknown & The building class 0 \\ \bottomrule
\end{tabular}
\vspace{-0.1cm}
\end{table*}

%% file: 05-Experiments.tex
\section{Experiments}
\label{sec:Experiments}

We conducted thorough experiments to evaluate the performance of \algname{}. The results are summarized as follows.
\begin{itemize}[leftmargin=10pt, noitemsep]
    \vspace{-0.1cm}  
    \item In the \emph{ten} high-dimensional and concept-drifted data sets, \algname{} outperforms the \emph{ten} state-of-the-art streaming anomaly detection algorithms with different approaches in terms of online anomaly detection accuracy (Section \ref{sec:overall_comparsion}) .
    \item \algname{} is able to detect exactly and adapt promptly to various types of real concept drifts (Section \ref{sec:drift_awareness}).
    \item The two main techniques employed in \algname{} are both highly effective to improve the accuracy (Section \ref{sec:ablation}).
    \item \algname{} is efficient and scalable with respect to varying input data rates and concept drift types (Section \ref{sec:efficiency_scalability}).
    \item \algname{} is robust to the variation of its own hyperparameter values, and their default values are valid (Section \ref{sec:sensitivity}).
    \vspace{-0.1cm}
\end{itemize}

\subsection{Experiment Setup}

\subsubsection{\textbf{Data Sets}} 
Table \ref{tbl:datasets} shows the summarized descriptions of ten \emph{high-dimensional}\,(complex) and \emph{concept-drifting}\,(evolving) benchmark data sets commonly used in other relevant literature as well. The synthetic data sets are generated from MNIST\,\cite{MNIST} and FMNIST\,\cite{F-MNIST} by simulating different concept drift types and durations, as widely used in high-dimensional anomaly detection and data stream classification\,\cite{RSRAE, HBP}. For the real data sets, the concept-drift types and durations are known in INSECTS\,\cite{INSECTS} but not in GAS\,\cite{GAS} and RIALTO\,\cite{RIALTO}. For all data sets, the ratio of anomalies over all data points was set to $1\%$. Refer to Appendix \ref{apdx:datasets} for more details.

\input{Figures/OverallPerformance.tex}
\input{Figures/ReliabilityTrend.tex}

\subsubsection{\textbf{Algorithms}}
We evaluated three instances of \algname{} implemented with three state-of-the-art AE-based anomaly detection algorithms: RAPP\,\cite{RAPP}, RSRAE\,\cite{RSRAE}, and DAGMM\,\cite{DAGMM}. Each of the three \algname{} instances is denoted with their base models (e.g., \algname{}$_{RAPP}$); the only difference among them is their constituent AE model. For a more comprehensive evaluation, the streaming variants of the three AE-based algorithms and two popular RNN-based anomaly detection algorithms\,(LSTM-ED\,\cite{LSTM-ED} and REBM\,\cite{REBM}) are prepared and respectively denoted as sRAPP, sRSRAE, sDAGMM, sLSTM-ED, and sREMB. At the same time,  five state-of-the-art streaming anomaly detection algorithms with different approaches (ensemble IF-based RRCF\,\cite{Sud16}, LSH-based MStream\,\cite{MSTREAM}, KDE-based STARE\,\cite{Yoon20}, and LOF-based MiLOF\,\cite{Sal16} and  DILOF\,\cite{Na18}) were also compared. Algorithm-specific hyperparameters were either set to the default values suggested by the authors or tuned by us to achieve the best accuracy. Details of the compared algorithms and hyperparameter settings are given in Appendix \ref{apdx:algorithms}. 

\smallskip
\subsubsection{\textbf{Performance Metrics}} Since the exact threshold of anomaly scores for verifying anomalies can vary across different applications and contexts, we used the Area Under Receiver Operating Characteristic\,(AUC) as the accuracy measure, which is widely used to evaluate anomaly detection\,\cite{Cam16}. For the deep learning-based algorithms, the ensemble-based RRCF, and the hashing-based MStream, the mean and the standard error of \emph{ten} repetitions, each with the different random initialization, were measured. The wall clock time for processing a batch of varying sizes was measured for efficiency and scalability tests.

\smallskip
\subsubsection{\textbf{Computing Platform}} All experiments were conducted on a Linux server with Intel Core i7-6700, 16GB RAM, and 1TB HDD. Ubuntu 16.04, Python 3.8, and TensorFlow 2.2 were installed. NVIDIA TITAN X was used for the deep learning algorithms.


\smallskip
\subsection{Overall Accuracy Comparison}
\label{sec:overall_comparsion}
We compared the AUC achieved by the three instances of \algname{} (i.e., \algname{}$_{RAPP}$, \algname{}$_{RSRAE}$, and \algname{}$_{DAGMM}$) with other algorithms for all data sets. The results are shown in Table \ref{tbl:overall_performance} (see Appendix \ref{apdx:sAE_results} for the results of sRAPP, sRSRAE, and sDAGMM). Evidently, the \algname{} instances achieved the highest accuracy in online anomaly detection for all data sets, regardless of the concept drift type. Specifically, the accuracy of \algname{} instances outperformed the best accuracy among the other algorithms by up to 36.9\%\,(for \algname{}$_{RAPP}$ in MNIST-GrdRec), and by 12.0\% on average over all data sets. The \algname{} instances improved the accuracy of the streaming variants of their base models by 10.7\% for sRAPP, 6.3\% for sRSRAE, and 9.7\% for sDAGMM when averaged over all data sets. This result clearly demonstrates the model-agnostic behavior and versatility of \algname{}.

\begin{figure}[t]
    \centering
    \includegraphics[width=\columnwidth]{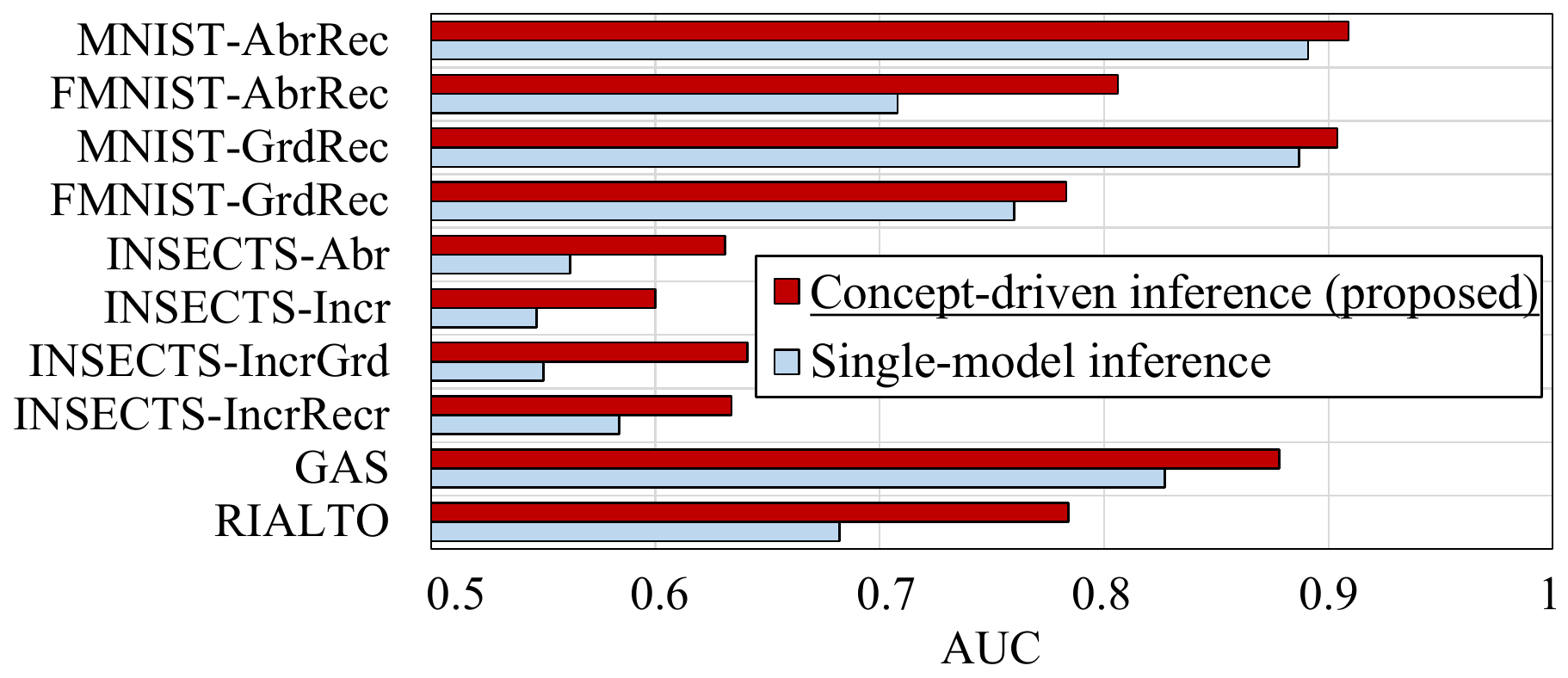}
    \vspace*{-0.7cm}
    \caption{Ablation study results of inference strategies.}
    \label{fig:ablation_inference}
\end{figure}

\subsection{Concept Drift Adaptation}
\label{sec:drift_awareness}

We tracked how the reliability of a model pool changes in real data streams with known concept drifts\,(INSECTS) where temperature changes refer to concept drifts. Figure \ref{fig:reliability_trend} shows the representative results of \algname{}$_{RSRAE}$ with the true trends of concepts and drift points. The model pool kept consistently high reliability throughout the entire data stream and the sudden drop points of the reliability were very close to the true abrupt drift points. This consistently high reliability was achieved using only five to eight models at most\,(3.9 models on average) in the model pool. This result means that \algname{} promptly detects and efficiently adapts to the real concept drifts so that it can achieve the highest anomaly detection accuracy with the minimum number of sufficient models. Interestingly, an unexpected drop was observed around the timestamp 7 in INSECTS-IncGrd in Figure \ref{fig:pool_adp_IncGrd}; it may have been caused by an additional unknown sudden drop of temperature during the incremental decrease or another type of concept drift\,(e.g., humidity change) that was overlooked.


\subsection{Ablation Study}
\label{sec:ablation}
We conducted ablation studies on the two main techniques used in \algname{}---the concept-driven inference and the concept drift-aware update. We present the results of \algname{}$_{RAPP}$ while the results from the other \algname{} instances showed similar patterns. For evaluating the efficacy of concept-driven inference, we prepared the variant of \algname{} employing the \emph{single-model inference} strategy, which uses only the most reliable model in the anomaly detection step. Figure \ref{fig:ablation_inference} shows that the concept-driven inference consistently improved the accuracy by up to 17\% over the single-model inference. This result confirms that it is worth considering various previous and ongoing concepts for calculating anomaly scores more exactly. 

For evaluating the efficacy of concept drift-aware update with the proposed \emph{similarity-based merge} strategy, we prepared the two variants of \algname{} respectively employing \emph{always-merge} strategy, which manages a single model by always merging with a new model, and the \emph{no-merge} strategy, which keeps all models without merging. As shown in Figure \ref{fig:ablation_merge}, the similarity-based merge achieved higher accuracy, by up to 17\% and 13\% over the always-merge and the no-merge, respectively. The similarity-based merge is more efficient than the no-merge by managing 37.6\% fewer models in the model pool, averaged over all data sets. These results show that keeping an adequate number of models with significantly different properties facilitates adapting to various types of concept drifts more accurately and efficiently.

\begin{figure}[t]
    \centering
    \includegraphics[width=\columnwidth]{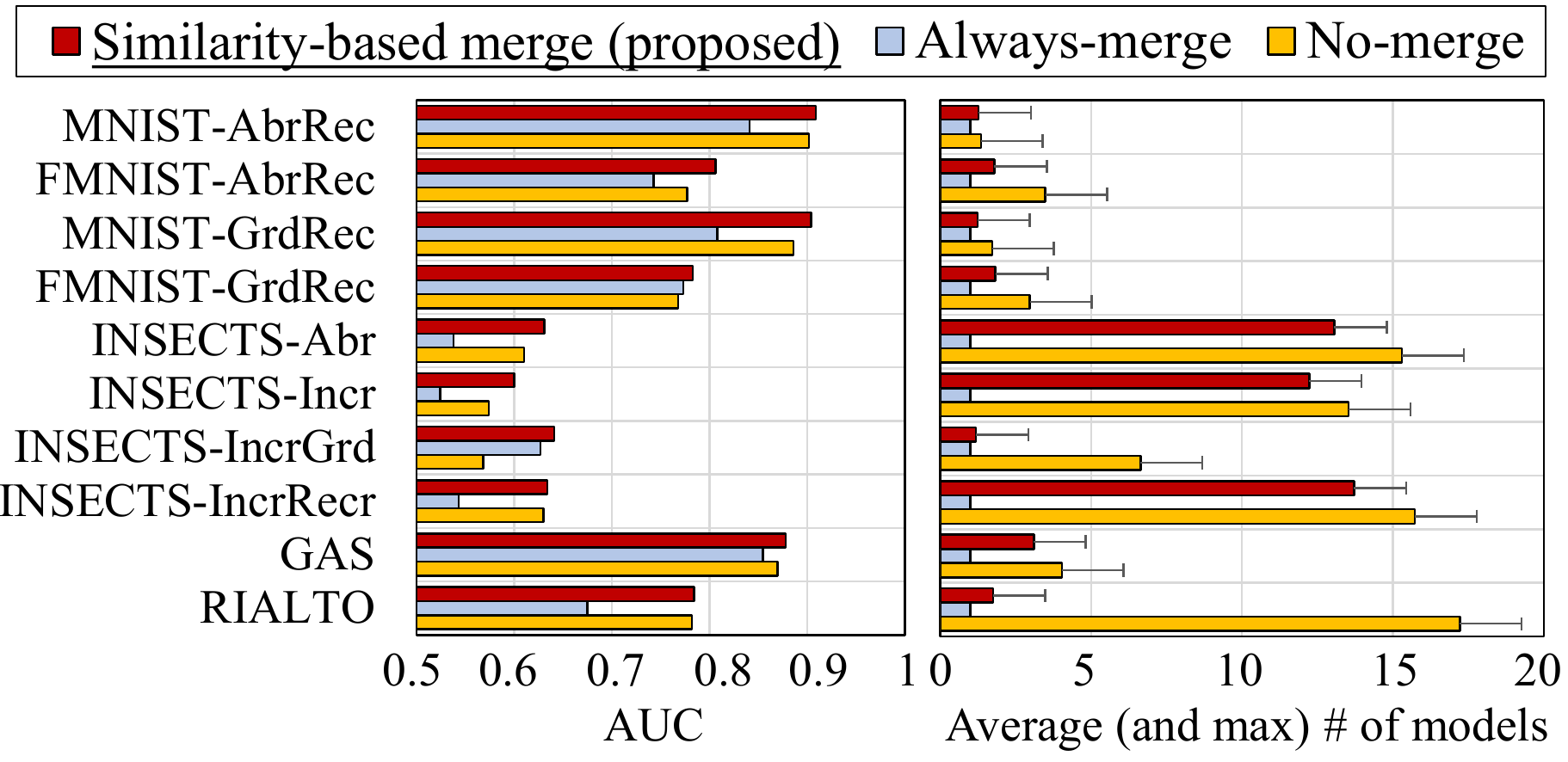}
    \caption{Ablation study results of merge strategies.}
    \label{fig:ablation_merge}
    \vspace*{-0.3cm}
\end{figure}


\subsection{Efficiency and Scalability}
\label{sec:efficiency_scalability}
We measured the processing time of \algname{}, simulating a real environment in which 128 to 2,048 devices emitting sensor-generated values. Figure \ref{fig:scalability} shows the representative results of \algname{}$_{RAPP}$, sRAPP, and RRCF (which is the ensemble-based algorithm and showed the strongest performance among the other algorithms). The other results are provided in Appendix \ref{apdx:scalability}. Note that the results from the other \algname{} instances and data sets showed similar patterns. Notably, \algname{} took less than a second for processing a batch of 512 data points, which would be sufficiently fast for a high-dimensional data stream from hundreds of devices processed by a single commodity machine. At the same time, \algname{} was consistently faster than the ensemble-based RRCF in all cases, which shows the merits of dynamically managed model pool compared with pre-fixed model ensembles. The increase rate of the processing time of \algname{} was comparable to those of the baselines. 

\begin{figure}[!t]
    \begin{subfigure}[b]{\columnwidth}
        \centering
        \includegraphics[width=0.7\textwidth]{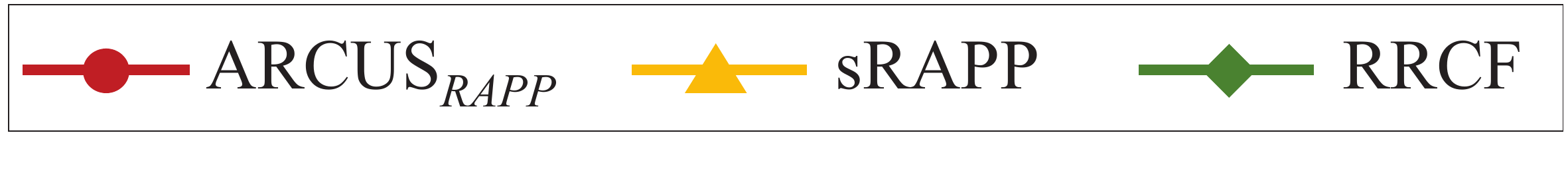}
    \end{subfigure}
    \begin{subfigure}[b]{0.325\columnwidth}
        \centering
        \includegraphics[width=\textwidth]{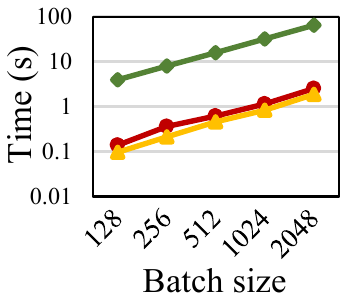}
    \end{subfigure}
    \begin{subfigure}[b]{0.325\columnwidth}
        \centering
        \includegraphics[width=\textwidth]{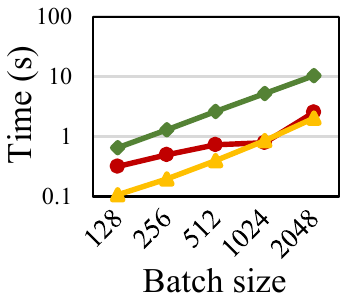}
    \end{subfigure}
    \begin{subfigure}[b]{0.325\columnwidth}
        \centering
        \includegraphics[width=\textwidth]{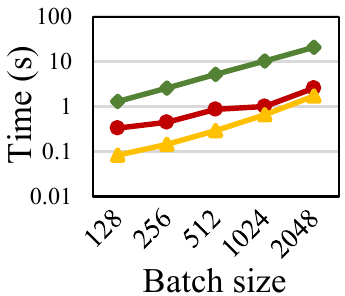}
    \end{subfigure}
    \text{\small \hspace{0.0cm} (a) MNIST-AbrRec.\hspace{0.6cm} (b) INSECTS-IncGrd. \hspace{0.8cm} (c) GAS\hspace{0.5cm}}
    \vspace*{-0.3cm}
    \caption{Scalability test results.}
    \label{fig:scalability}
\end{figure}

\begin{figure}[!t]
    \centering
    \includegraphics[width=\columnwidth]{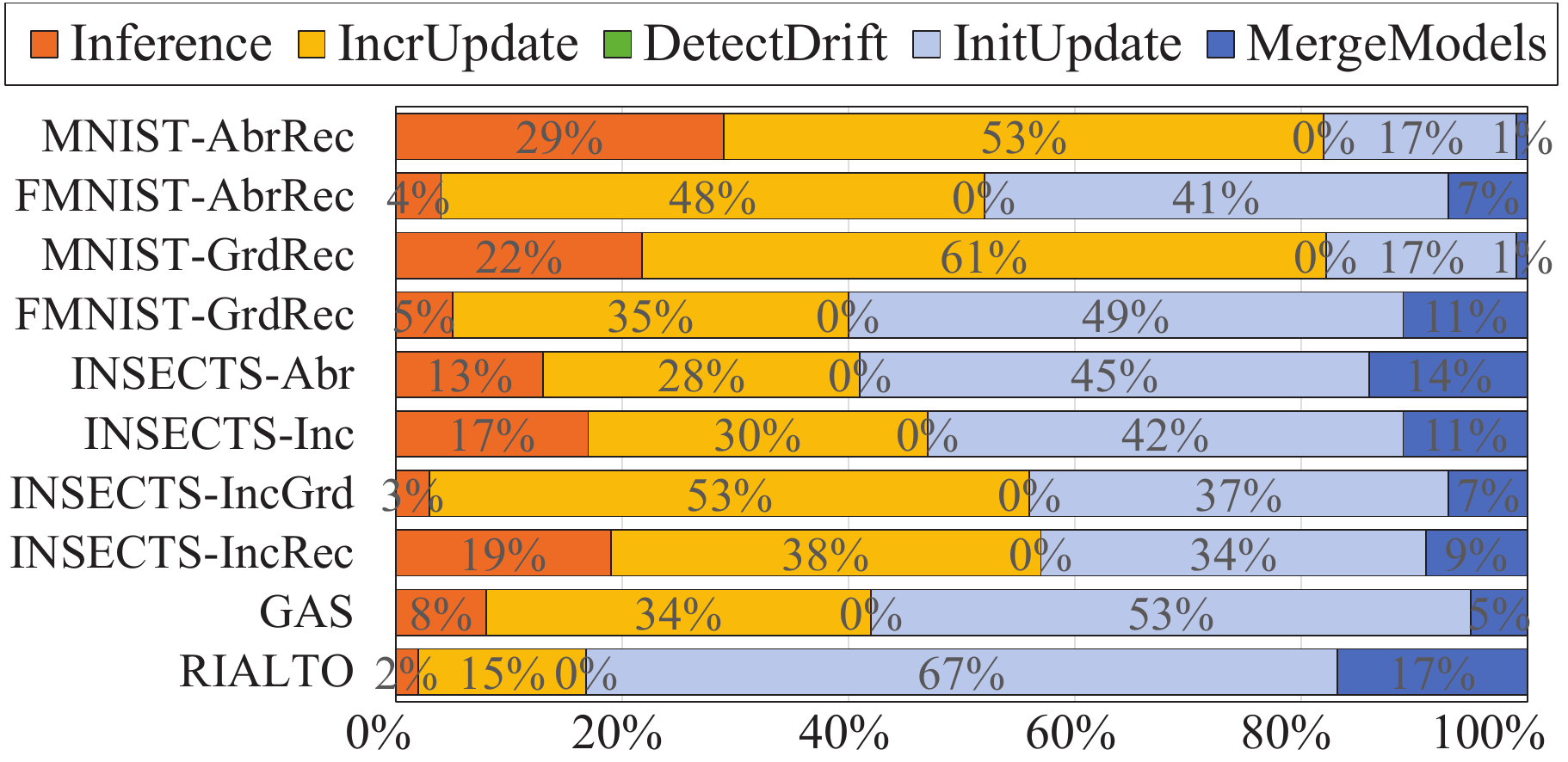}
    \caption{Processing time breakdown results.}
    \label{fig:breakdown}
    \vspace*{-0.6cm}
\end{figure}

Figure \ref{fig:breakdown} shows the breakdown of processing time of \algname{}$_{RAPP}$ into five steps for each data set. (The results of the other \algname{} instances were similar.) The five steps correspond to Line 4, Lines 8--9, Lines 6--7, Line 11, and Line 12 in Algorithm 1 in the paper. When averaged over all data sets, the processing time for inference and incremental model update was similar to that for the model pool adaptation consisting of drift detection, initial model update, and model merge. Except for the initial model update, which can be flexibly controlled by the user-provided number of initialization epochs, the remaining model pool adaptation\,(i.e., drift detection and model merging) took only 8.2\% of the total processing time. Interestingly, the specific proportions of each step vary across the data sets since the behavior of \algname{} is affected by the characteristics of the data stream. For instance, the model pool adaptation of MNIST data sets took less than 20\% of the total processing time, while those of GAS and RIALTO took more than 50\% of the total processing time. Overall, \algname{} tended to adapt more frequently in a data stream with unknown and implicit concept drifts\,(e.g., GAS and RIATLO) than a data stream with known and explicit concept drifts\,(e.g., MNIST).

\begin{figure}
    \centering
    \includegraphics[width=\columnwidth]{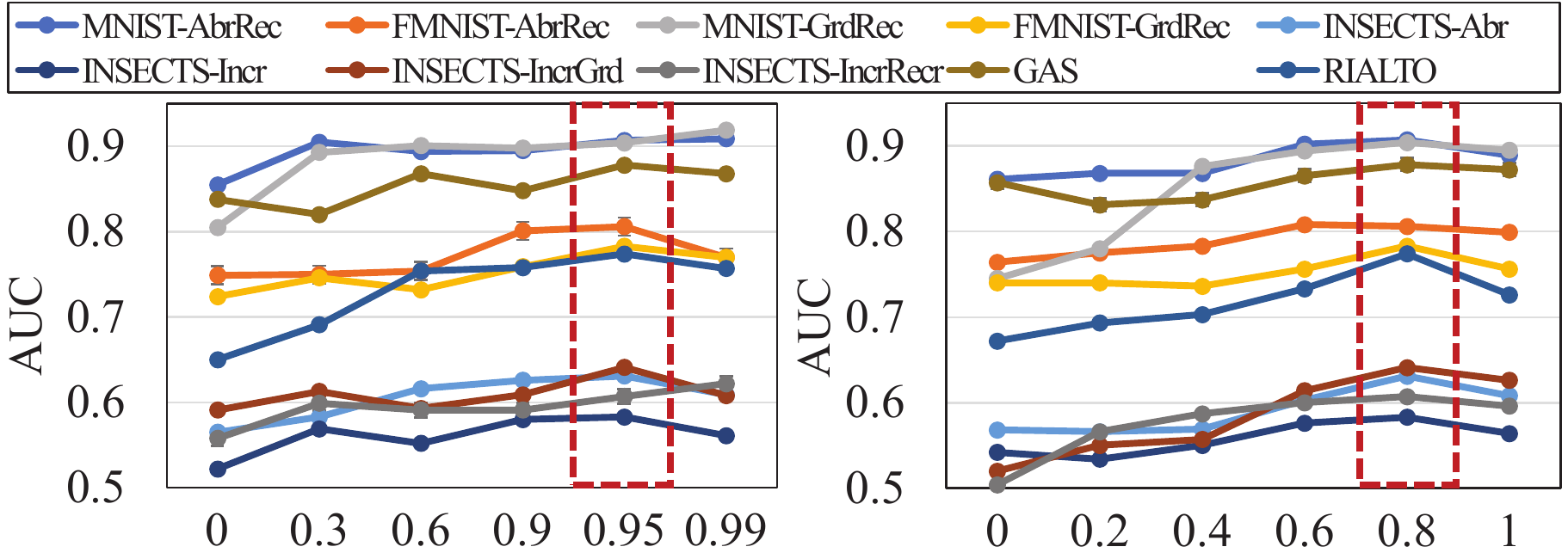}
    \text{\small \quad\quad (a) Reliability threshold\,($\alpha$).\quad\quad\quad (b) Similarity threshold\,($\gamma$).}
    \vspace{-0.2cm}
    \caption{Sensitivity analysis results.}
    \label{fig:sensitivity}
\end{figure}

\subsection{Parameter Sensitivity Analysis}
\label{sec:sensitivity}
We conducted sensitivity analysis on the two main parameters used in the concept drift-aware update step: the reliability threshold $\alpha$ (to trigger model pool update) and the similarity threshold $\gamma$ (to merge models in a model pool). Due to space limitation, we present the results of \algname{}$_{RAPP}$ while the results from the other \algname{} instances showed similar patterns. For $\alpha$, Figure \ref{fig:sensitivity}a shows that AUC peaks at or nearly converges around the default value, 0.95. An extremely high reliability threshold may help improve the accuracy marginally, but it will damage the efficiency because of too frequent pool updates. For $\gamma$, Figure \ref{fig:sensitivity}b shows that AUC peaks at or nearly converges around the default value, 0.8. As mentioned earlier, the similarity of models trained under the same concept was also usually higher than 0.8, which is also demonstrated in Figure \ref{fig:cdf_similarity} in Appendix \ref{apdx:algorithms}). Likewise, an extremely high similarity threshold may improve accuracy marginally, but it will incur too many models in a model pool. 


%% file: Figures/OverallPerformance.tex
\bgroup
\def\arraystretch{0.4}%
\begin{table*}
\center
\small
\caption{Overall performance comparison. The highest AUC results in each data set are marked in underlined bold.}
\vspace{-0.3cm}
\label{tbl:overall_performance}
\begin{tabular}[c]{@{}C{1.2cm}C{2.2cm}|C{1.0cm}C{1.0cm}C{1.2cm}|C{1.38cm}C{1.10cm}C{0.8cm}C{0.9cm}C{0.9cm}C{0.9cm}C{1.0cm}@{}}
\toprule
\multicolumn{1}{c}{} &\multicolumn{1}{c}{} & \multicolumn{3}{c}{\textbf{\algname{}$_{\{\emph{\text{base model}}\}}$}}
& \multicolumn{7}{c}{Streaming anomaly detection algorithms}     \\ 
\multicolumn{1}{l}{}  &  Data set    &$\textit{RAPP}$ & $\textit{RSRAE}$   & $\textit{DAGMM}$    & sLSTM-ED   & sREBM  & STARE & RRCF & MiLOF & DILOF & MStream  \\ \toprule
\multirow{11.5}{*}{{Synthetic}}                                                         & \begin{tabular}[c]{@{}c@{}}MNIST-AbrRec\end{tabular}         & \begin{tabular}[c]{@{}c@{}}\textbf{\underline{0.909}}\\ ($\pm0.004$)\end{tabular} & \begin{tabular}[c]{@{}c@{}}0.831\\ ($\pm0.003$)\end{tabular} & \begin{tabular}[c]{@{}c@{}}0.747\\ ($\pm0.009$)\end{tabular} & \begin{tabular}[c]{@{}c@{}}0.662\\ ($\pm0.006$)\end{tabular} & \begin{tabular}[c]{@{}c@{}}0.581\\ ($\pm0.001$)\end{tabular} &
\begin{tabular}[c]{@{}c@{}}0.574\\ ($\pm0$)\end{tabular}
&
\begin{tabular}[c]{@{}c@{}}0.711\\ ($\pm0.006$)\end{tabular} &
\begin{tabular}[c]{@{}c@{}}0.460\\ ($\pm0$)\end{tabular} &  \begin{tabular}[c]{@{}c@{}}0.655\\ ($\pm0$)\end{tabular} & \begin{tabular}[c]{@{}c@{}}0.491\\ ($\pm0.002$)\end{tabular}
\\ \cdashline{2-12}
& \begin{tabular}[c]{@{}c@{}}FMNIST-AbrRec\end{tabular}       & \begin{tabular}[c]{@{}c@{}}\textbf{\underline{0.806}}\\ ($\pm0.001$)\end{tabular} & \begin{tabular}[c]{@{}c@{}}0.743\\ ($\pm0.006$)\end{tabular} & \begin{tabular}[c]{@{}c@{}}0.657\\ ($\pm0.004$)\end{tabular} & \begin{tabular}[c]{@{}c@{}}0.772\\ ($\pm0.004$)\end{tabular} & \begin{tabular}[c]{@{}c@{}}0.603\\ ($\pm0.003$)\end{tabular} &
\begin{tabular}[c]{@{}c@{}}0.576\\ ($\pm0$)\end{tabular}
&
\begin{tabular}[c]{@{}c@{}}0.713\\ ($\pm0.013$)\end{tabular} &
\begin{tabular}[c]{@{}c@{}}0.434\\ ($\pm0$)\end{tabular} & \begin{tabular}[c]{@{}c@{}}0.513\\ ($\pm0$)\end{tabular} & \begin{tabular}[c]{@{}c@{}}0.717\\ ($\pm0.002$)\end{tabular} \\  \cdashline{2-12}

& \begin{tabular}[c]{@{}c@{}}MNIST-GrdRec\end{tabular}       & \begin{tabular}[c]{@{}c@{}}\textbf{\underline{0.904}}\\ ($\pm0.011$)\end{tabular} & \begin{tabular}[c]{@{}c@{}}0.784\\ ($\pm0.009$)\end{tabular} & \begin{tabular}[c]{@{}c@{}}0.707\\ ($\pm0.002$)\end{tabular} &  \begin{tabular}[c]{@{}c@{}}0.622\\ ($\pm0.003$)\end{tabular} & \begin{tabular}[c]{@{}c@{}}0.502\\ ($\pm0.002$)\end{tabular} &
\begin{tabular}[c]{@{}c@{}}0.574\\ ($\pm0$)\end{tabular}
&
\begin{tabular}[c]{@{}c@{}}0.660\\ ($\pm0.007$)\end{tabular} &
\begin{tabular}[c]{@{}c@{}}0.460\\ ($\pm0$)\end{tabular} & \begin{tabular}[c]{@{}c@{}}0.649\\ ($\pm0$)\end{tabular} & \begin{tabular}[c]{@{}c@{}}0.632\\ ($\pm0.004$)\end{tabular} \\ \cdashline{2-12}

& \begin{tabular}[c]{@{}c@{}}FMNIST-GrdRec\end{tabular}       & \begin{tabular}[c]{@{}c@{}}\textbf{\underline{0.783}}\\ ($\pm0.012$)\end{tabular} & \begin{tabular}[c]{@{}c@{}}0.682\\ ($\pm0.004$)\end{tabular} & \begin{tabular}[c]{@{}c@{}}0.652\\ ($\pm0.004$)\end{tabular} & \begin{tabular}[c]{@{}c@{}}0.630\\ ($\pm0.008$)\end{tabular} & \begin{tabular}[c]{@{}c@{}}0.485\\ ($\pm0.003$)\end{tabular} & \begin{tabular}[c]{@{}c@{}}0.566\\ ($\pm0$)\end{tabular}
&
\begin{tabular}[c]{@{}c@{}}0.730\\ ($\pm0.010$)\end{tabular} &
\begin{tabular}[c]{@{}c@{}}0.494\\ ($\pm0$)\end{tabular} & \begin{tabular}[c]{@{}c@{}}0.510\\ ($\pm0$)\end{tabular} & \begin{tabular}[c]{@{}c@{}}0.497\\ ($\pm0.003$)\end{tabular} \\ \midrule

\multirow{11.5}{*}{\begin{tabular}[c]{@{}c@{}} Real\\ (known \\drifts)\end{tabular}}   & \begin{tabular}[c]{@{}c@{}}INSECTS-Abr\end{tabular}    & \begin{tabular}[c]{@{}c@{}}0.631\\ ($\pm0.009$)\end{tabular} & \begin{tabular}[c]{@{}c@{}}\textbf{\underline{0.814}}\\ ($\pm0.006$)\end{tabular} & \begin{tabular}[c]{@{}c@{}}0.652\\ ($\pm0.018$)\end{tabular} & \begin{tabular}[c]{@{}c@{}}0.749\\ ($\pm0.008$)\end{tabular} & \begin{tabular}[c]{@{}c@{}}0.471\\ ($\pm0.001$)\end{tabular} & \begin{tabular}[c]{@{}c@{}}0.555\\ ($\pm0$)\end{tabular}
&
\begin{tabular}[c]{@{}c@{}}0.695\\ ($\pm0.018$)\end{tabular} &
\begin{tabular}[c]{@{}c@{}}0.393\\ ($\pm0$)\end{tabular} & \begin{tabular}[c]{@{}c@{}}0.730\\ ($\pm0$)\end{tabular} & \begin{tabular}[c]{@{}c@{}}0.709\\ ($\pm0.015$)\end{tabular} \\ \cdashline{2-12}


   & \begin{tabular}[c]{@{}c@{}}INSECTS-Inc\end{tabular}    & \begin{tabular}[c]{@{}c@{}}0.600\\ ($\pm0.004$)\end{tabular} & \begin{tabular}[c]{@{}c@{}}\textbf{\underline{0.794}}\\ ($\pm0.001$)\end{tabular} & \begin{tabular}[c]{@{}c@{}}0.572\\ ($\pm0.034$)\end{tabular}  &  \begin{tabular}[c]{@{}c@{}}0.696\\ ($\pm0.004$)\end{tabular} & \begin{tabular}[c]{@{}c@{}}0.383\\ ($\pm0.001$)\end{tabular} & 
\begin{tabular}[c]{@{}c@{}}0.559\\ ($\pm0$)\end{tabular}
&
   \begin{tabular}[c]{@{}c@{}}0.669\\ ($\pm0.011$)\end{tabular} &
   \begin{tabular}[c]{@{}c@{}}0.415\\ ($\pm0$)\end{tabular} & \begin{tabular}[c]{@{}c@{}}0.757\\ ($\pm0$)\end{tabular} & \begin{tabular}[c]{@{}c@{}}0.593\\ ($\pm0.001$)\end{tabular}
   \\ \cdashline{2-12}
   
   & \begin{tabular}[c]{@{}c@{}}INSECTS-IncGrd\end{tabular} & \begin{tabular}[c]{@{}c@{}}0.641\\ ($\pm0.040$)\end{tabular} & \begin{tabular}[c]{@{}c@{}}\textbf{\underline{0.845}}\\ ($\pm0.002$)\end{tabular} & \begin{tabular}[c]{@{}c@{}}0.658\\ ($\pm0.028$)\end{tabular} & \begin{tabular}[c]{@{}c@{}}0.795\\ ($\pm0.005$)\end{tabular} & \begin{tabular}[c]{@{}c@{}}0.575\\ ($\pm0.015$)\end{tabular} & 
   \begin{tabular}[c]{@{}c@{}}0.594\\ ($\pm0$)\end{tabular}
   &
   \begin{tabular}[c]{@{}c@{}}0.719\\ ($\pm0.032$)\end{tabular} &
   \begin{tabular}[c]{@{}c@{}}0.395\\ ($\pm0$)\end{tabular} & \begin{tabular}[c]{@{}c@{}}0.746\\ ($\pm0$)\end{tabular} & \begin{tabular}[c]{@{}c@{}}0.628\\ ($\pm0.001$)\end{tabular}
   \\ \cdashline{2-12}

   & \begin{tabular}[c]{@{}c@{}}INSECTS-IncRec\end{tabular} & \begin{tabular}[c]{@{}c@{}}0.634\\ ($\pm0.012$)\end{tabular} & \begin{tabular}[c]{@{}c@{}}\textbf{\underline{0.811}}\\ ($\pm0.001$)\end{tabular}  & \begin{tabular}[c]{@{}c@{}}0.667\\ ($\pm0.001$)\end{tabular}  & \begin{tabular}[c]{@{}c@{}}0.709\\ ($\pm0.005$)\end{tabular} & \begin{tabular}[c]{@{}c@{}}0.491\\ ($\pm0.006$)\end{tabular} & 
   \begin{tabular}[c]{@{}c@{}}0.551\\ ($\pm0$)\end{tabular}
   &
   \begin{tabular}[c]{@{}c@{}}0.680\\ ($\pm0.003$)\end{tabular} &
   \begin{tabular}[c]{@{}c@{}}0.381\\ ($\pm0$)\end{tabular} & \begin{tabular}[c]{@{}c@{}}0.743\\ ($\pm0$)\end{tabular} & \begin{tabular}[c]{@{}c@{}}0.637\\ ($\pm0.001$)\end{tabular} \\ \midrule


\multirow{1}{*}{\vspace*{1.0cm}\begin{tabular}[c]{@{}c@{}} Real\\(unknown \\ drift) \end{tabular}} & GAS     & \begin{tabular}[c]{@{}c@{}}\textbf{\underline{0.878}}\\ ($\pm0.008$)\end{tabular} & \begin{tabular}[c]{@{}c@{}}0.573\\ ($\pm0.017$)\end{tabular} & \begin{tabular}[c]{@{}c@{}}0.545\\ ($\pm0.039$)\end{tabular} & \begin{tabular}[c]{@{}c@{}}0.408\\ ($\pm0.005$)\end{tabular} & \begin{tabular}[c]{@{}c@{}}0.506\\ 
($\pm0.002$)\end{tabular} & 
\begin{tabular}[c]{@{}c@{}}0.635\\ ($\pm0$)\end{tabular}
&
\begin{tabular}[c]{@{}c@{}}0.804\\ ($\pm0.010$)\end{tabular} &
\begin{tabular}[c]{@{}c@{}}0.589\\ ($\pm0$)\end{tabular} & \begin{tabular}[c]{@{}c@{}}0.470\\ ($\pm0$)\end{tabular} & \begin{tabular}[c]{@{}c@{}}0.480\\ ($\pm0.001$)\end{tabular} \\ \cdashline{2-12}

   & RIALTO        & \begin{tabular}[c]{@{}c@{}}\textbf{\underline{0.784}}\\ ($\pm0.015$)\end{tabular} & \begin{tabular}[c]{@{}c@{}}0.683\\ ($\pm0.020$)\end{tabular} & \begin{tabular}[c]{@{}c@{}}0.562\\ ($\pm0.020$)\end{tabular} & \begin{tabular}[c]{@{}c@{}}0.617\\ ($\pm0.010$)\end{tabular}  & \begin{tabular}[c]{@{}c@{}}0.492\\ ($\pm0.001$)\end{tabular} & 
   \begin{tabular}[c]{@{}c@{}}0.532\\ ($\pm0$)\end{tabular}
   &
   \begin{tabular}[c]{@{}c@{}}0.731\\ ($\pm0.003$)\end{tabular} &
   \begin{tabular}[c]{@{}c@{}}0.456\\ ($\pm0$)\end{tabular} & \begin{tabular}[c]{@{}c@{}}0.742\\ ($\pm0$)\end{tabular} & \begin{tabular}[c]{@{}c@{}}0.699\\ ($\pm0.001$)\end{tabular} \\ \bottomrule
\end{tabular}
\vspace{-0.1cm}
\end{table*}
\egroup

%% file: Figures/ReliabilityTrend.tex
\begin{figure*}[!t]
    \begin{subfigure}[b]{\textwidth}
        \centering
        \includegraphics[width=0.9\textwidth]{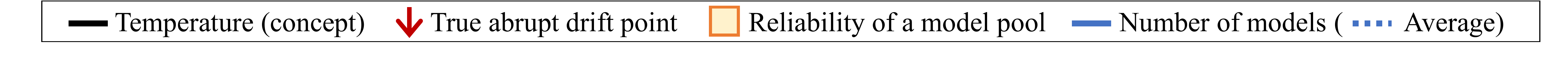}
        \vspace{-0.2cm}
        
    \end{subfigure}
    
    \begin{subfigure}[b]{0.24\textwidth}
        \centering
        \includegraphics[width=\textwidth]{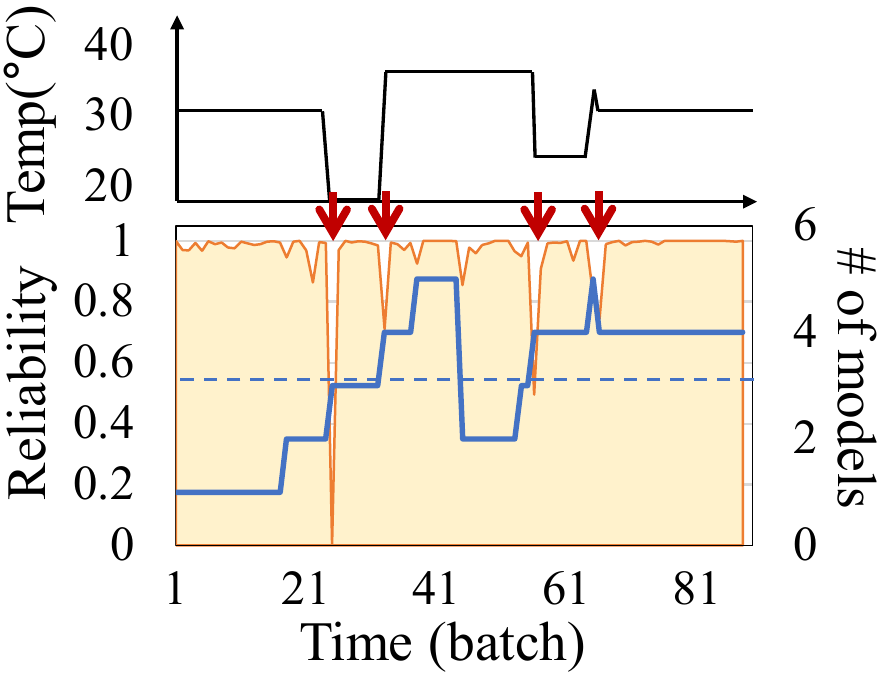}
        \vspace{-0.5cm}
        \caption{INSECTS-Abr.}
    \end{subfigure}
    \begin{subfigure}[b]{0.24\textwidth}
        \centering
        \includegraphics[width=\textwidth]{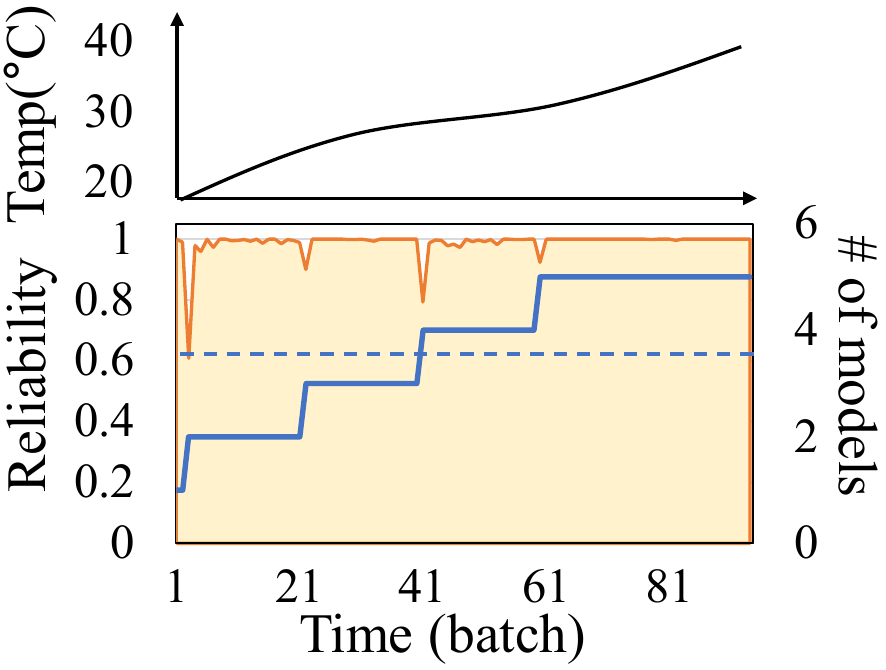}
        \vspace{-0.5cm}
        \caption{INSECTS-Inc.}
    \end{subfigure}
    \begin{subfigure}[b]{0.24\textwidth}
        \centering
        \includegraphics[width=\textwidth]{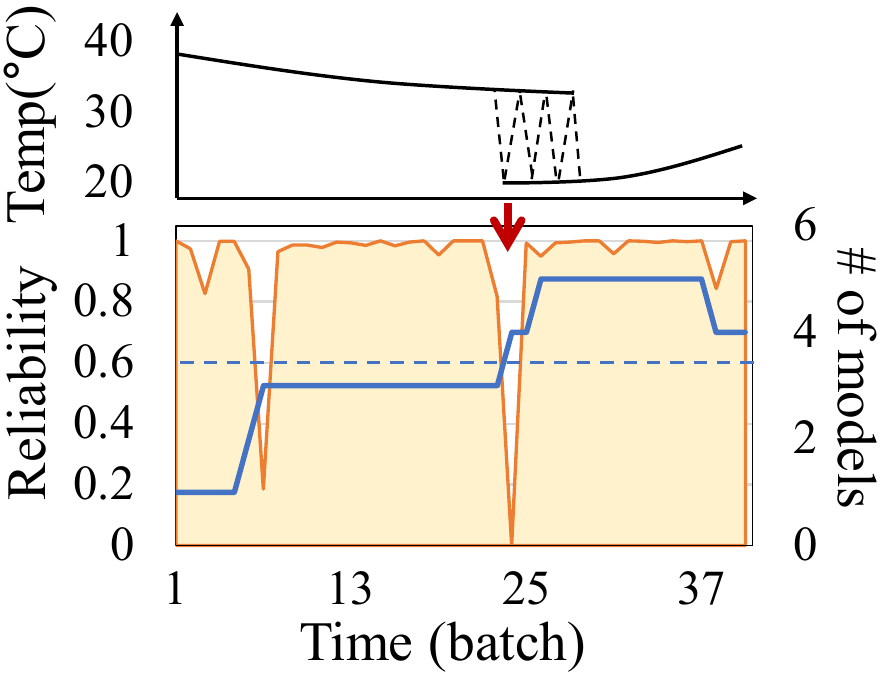}
        \vspace{-0.5cm}
        \caption{INSECTS-IncGrd.}
        \label{fig:pool_adp_IncGrd}
    \end{subfigure}
    \begin{subfigure}[b]{0.24\textwidth}
        \centering
        \includegraphics[width=\textwidth]{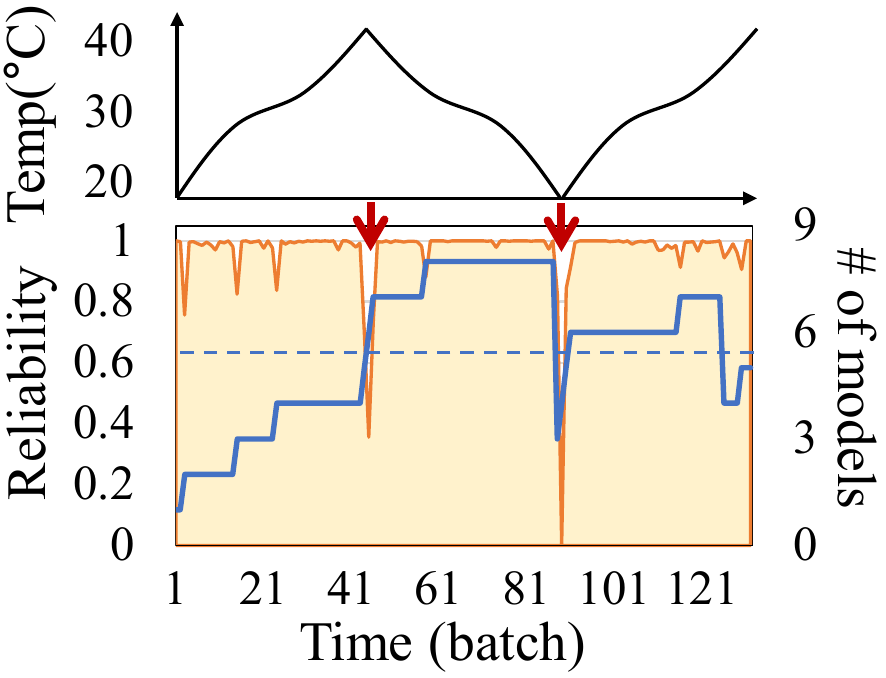}
        \vspace{-0.5cm}
        \caption{INSECTS-IncRec.}
    \end{subfigure}    
    \vspace{-0.35cm}
    \caption{The trends of the model pool reliability estimated by \algname{} in real concept-drifted data streams.}
    \label{fig:reliability_trend}
    \vspace{-0.4cm}
\end{figure*}

%% file: 06-Conclusion.tex
\section{Conclusion}
\label{sec:Conclusion}
This paper proposed \algname{}, a framework for online deep anomaly detection which can be instantiated with any AE-based deep anomaly detection method. \algname{} is specialized to handle complex evolving data streams by the \emph{adaptive model pooling} approach with two main techniques---\emph{concept-driven inference} and \emph{concept drift-aware update}. In comprehensive experiments using ten data sets, \algname{} outperformed state-of-the-art streaming anomaly detection methods by up to 37\% in accuracy. Much of this performance advantage is attributed to the versatile modeling power from model pooling. Overall, we believe that our work opens a new possibility in online anomaly detection research.

There are interesting directions \algname{} can be further developed into. 
First, the deep anomaly-detection model used to instantiate \algname{} can be extended beyond the AE. While we chose the AE as the default model to materialize the adaptive model pooling approach because of its structural simplicity and unsupervised learning mechanism, other models (e.g., GAN- or RNN-based) can be used under the same approach. The concept of model similarity, however, should be carefully tailored to the specific model since the latent representation similarity proposed for AE-based models may not always be applicable.
Second, a model pool can adopt other adaptation strategies than the model initialization from scratch or the model merging based on federated learning. For instance, the initialization and update of a model can be facilitated by the incorporation of shared common knowledge of models through distillation or regularization strategies based on continual learning or transfer learning.
Third, a semi-supervised approach can help understand the dynamics of a model pool and optimize it for different concept drifts observed in a data stream. A few human-provided labels or domain knowledge of anomalies or concept drifts are useful to derive the reliability threshold or the similarity threshold which are the main tunable hyperparameters in \algname{}.

%% file: 07-Appendix.tex
\clearpage
\appendix

\section{Supplementary Material}
\subsection{Data Sets}
\label{apdx:datasets}
\begin{itemize}[leftmargin=10pt, noitemsep]
    \item \ul{Synthetic data sets}: MNIST\,\cite{MNIST} and FMNIST\,\cite{F-MNIST} are handwritten digit images and fashion item images, respectively, and are widely used to \emph{simulate} anomaly detection scenarios with a high complexity\,\cite{RDA, RAPP, RSRAE} or data streams\,\cite{HBP, IADM}.
    To represent a concept, we randomly set certain digits of item classes out of ten classes as anomaly target classes and the other ones as normal classes.  
    The duration of each concept was varied randomly between one to four times the batch size. Two types of concept drifts were simulated: ``abrupt and recurrent'' (MNIST-AbrRec, FMNIST-AbrRec) and ``gradual and recurrent'' (MNIST-GrdRec, FMNIST-GrdRec).
    \item \ul{Real data sets\,(with known drifts)}: INSECTS\,\cite{INSECTS} is a real concept-drifting benchmark data set containing optical sensor values collected while monitoring flying insects\,(e.g., mosquitos). Concepts are controlled by changing the temperature level, which affects the flying behaviors of insects. The southern house mosquito, which transmits zoonotic diseases, was chosen as an anomaly target class. We used four types of INSECTS data sets with different concept drift types---``abrupt'' (INSECTS-Abr), ``incremental'' (INSECTS-Inc), ``incremental and gradual'' (INSECTS-IncGrd), and ``incremental and recurrent'' (INSECTS-IncRec).
    \item \ul{Real data sets\,(with unknown drifts)}: GAS\,\cite{GAS} is a data set gathered in a gas delivery platform for 36 months and contains sensor values from monitoring six types of pure gaseous substances. A different gas concentration refers to a concept where the true drift types and timings are unknown. Acetaldehyde, a toxic chemical, was chosen as an anomaly target class. RIALTO\,\cite{RIALTO} contains normalized RGB encodings from 20 days of video recordings of ten buildings around the Rialto bridge in Venice. The building class 0 was set as an anomaly target class. Weather and lighting conditions refer to concepts as they directly affect the video recording results, while the true types and timings of concept drifts are unknown. 
\end{itemize}

\smallskip
\subsection{Algorithms and Hyperparameter Settings}
\label{apdx:algorithms}
\subsubsection{\textbf{\algname{} and AE-based Algorithms}}
We used the three AE-based state-of-the-art anomaly detection algorithms: RAPP\,\cite{RAPP}\footnote{RAPP: \url{https://github.com/Aiden-Jeon/RaPP}}, RSRAE\,\cite{RSRAE}\footnote{RSRAE: \url{https://github.com/dmzou/RSRAE}}, and DAGMM\,\cite{DAGMM}\footnote{DAGMM: \url{https://github.com/tnakae/DAGMM}}. They are all based on a basic AE and its variants (e.g., variational AE or denoising AE) but use different techniques for anomaly detection---hidden reconstruction errors along the projection pathway of encoding and decoding layers by RAPP, a linear transformation layer in the latent space by RSRAE, and the Gaussian mixture model added to the AE by DAGMM. We implemented the three algorithms based on the source code publicly available.  

Each of the three algorithms was used for (i) a straightforward streaming variant and (ii) an AE-based model $M$ of \algname{}. For the first purpose, we re-trained an AE model incrementally whenever a new batch is received. The streaming variants of them are referred to sRAPP, sRSRAE, and sDAGMM, respectively. For the second purpose (i.e., to incorporate each model into \algname{}), we implemented the interface for creating a new model, executing training epochs with incoming batches, obtaining anomaly scores, extracting latent representations, and merging with existing models. As mentioned in the main paper, the reliability threshold and the similarity threshold for \algname{} were fixed to 0.95\,(the commonly-used statistical significance threshold) and 0.8\,(the empirically-confirmed threshold in Figure \ref{fig:cdf_similarity}), respectively.

\begin{figure}[!t]
\centering
    \begin{subfigure}[b]{\columnwidth}
        \includegraphics[width=\columnwidth]{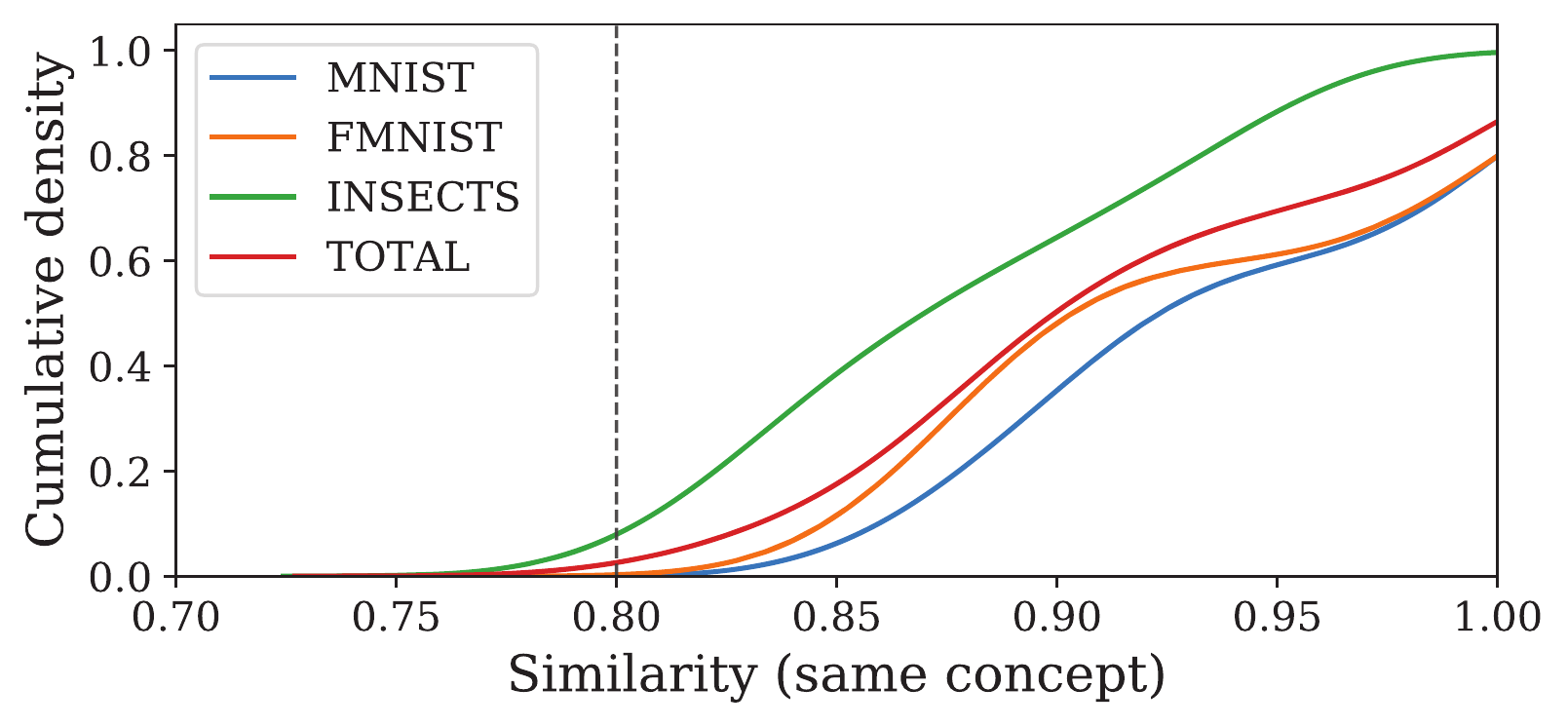}
    \end{subfigure}
    \begin{subfigure}[b]{\columnwidth}
        \includegraphics[width=\columnwidth]{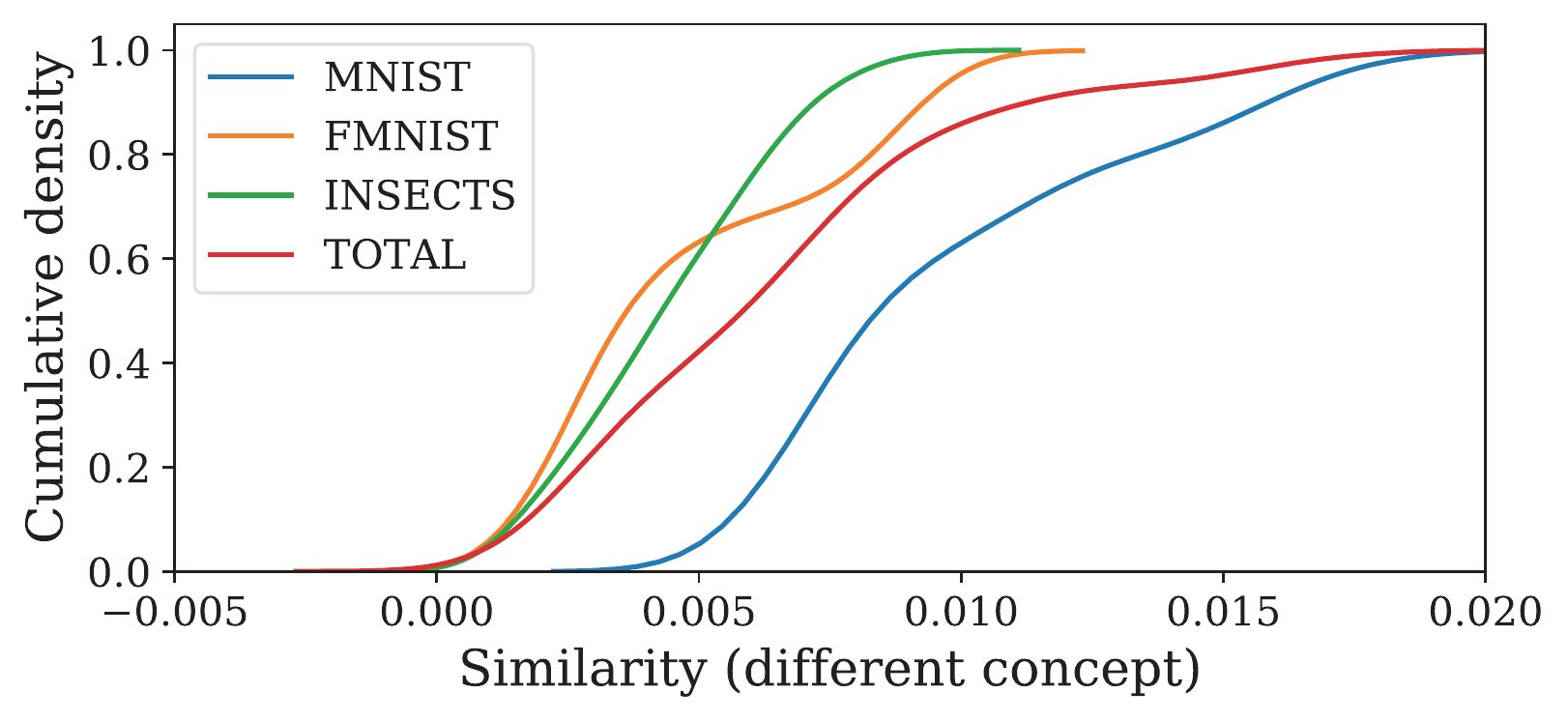}
    \end{subfigure}    
    \vspace{-0.4cm}
    \caption{For three data streams with known concepts, we analyzed cumulative density as a function of the latent representation similarities between AE-based models trained under the same concept (top) or the different concepts (bottom). The similarity values between the models trained under the same concepts are mostly above 0.8, while those under the different concepts are at most 0.02.}
    \label{fig:cdf_similarity}
\end{figure}

\begin{figure*}[!t]
    \begin{subfigure}[b]{\textwidth}
        \centering
        \captionsetup{justification=centering}
        \includegraphics[width=0.35\textwidth]{Figures/Scalability_legend.pdf}
    \end{subfigure}
    \begin{subfigure}[b]{0.14\textwidth}
        \centering
        \captionsetup{justification=centering}
        \includegraphics[width=\textwidth]{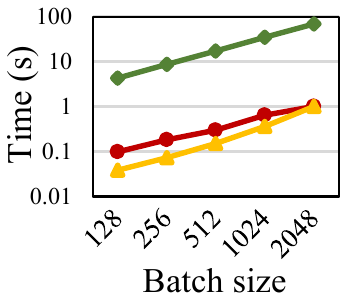}
        \caption{FMNIST\newline-AbrRec.}
    \end{subfigure}
    \hspace{-0.1cm}
    \begin{subfigure}[b]{0.14\textwidth}
        \centering
        \captionsetup{justification=centering}
        \includegraphics[width=\textwidth]{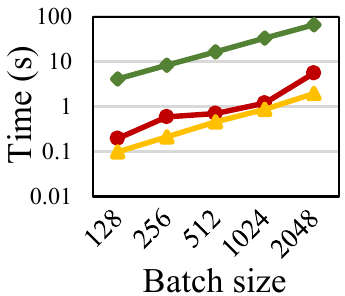}
        \caption{MNIST\newline-GrdRec.}
    \end{subfigure}
    \hspace{-0.1cm}
    \begin{subfigure}[b]{0.14\textwidth}
        \centering
        \captionsetup{justification=centering}
        \includegraphics[width=\textwidth]{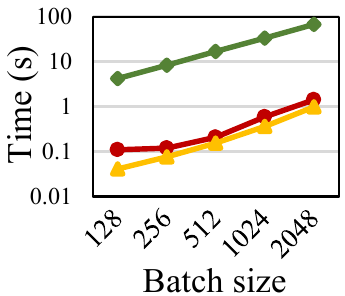}
        \caption{FMNIST\newline-GrdRec.}
    \end{subfigure}
    \hspace{-0.1cm}
    \begin{subfigure}[b]{0.14\textwidth}
        \centering
        \captionsetup{justification=centering}
        \includegraphics[width=\textwidth]{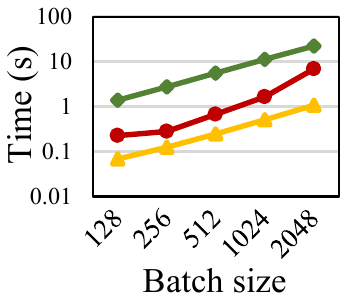}
        \caption{INSECTS\newline-Abr.}
    \end{subfigure}
    \hspace{-0.1cm}
    \begin{subfigure}[b]{0.14\textwidth}
        \centering
        \captionsetup{justification=centering}
        \includegraphics[width=\textwidth]{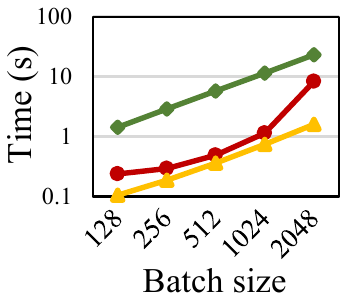}
        \caption{INSECTS\newline-Inc.}
    \end{subfigure}
    \hspace{-0.1cm}
    \begin{subfigure}[b]{0.14\textwidth}
        \centering
        \captionsetup{justification=centering}
        \includegraphics[width=\textwidth]{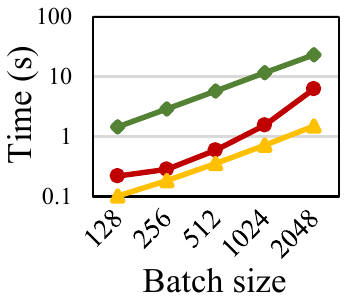}
        \caption{INSECTS\newline-IncRec.\!}
    \end{subfigure}
    \hspace{-0.1cm}
    \begin{subfigure}[b]{0.14\textwidth}
        \centering
        \captionsetup{justification=centering}
        \includegraphics[width=\textwidth]{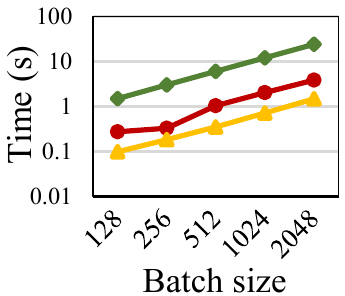}
        \caption{RIALTO.\newline}
    \end{subfigure}
    \vspace{-0.1cm}
    \caption{Scalability test results for all the data sets omitted in the paper.}
    \vspace{-0.1cm}
    \label{fig:scalability_others}
\end{figure*}

We set the number of epochs for updating the AE-based algorithms to 5 for the initial update and 1 for the incremental update. The batch size was set to 512 (with the mini-batch size 32). The size of a latent space in AE-based models was set to the number of principal components, which guarantees at least 70\% of the variance to be explained, following the relevant work\,\cite{RAPP}. The optimal number of layers\,(i.e., the depth of the encoder\,/\,decoder) and learning rate were tuned to achieve the best result from 2 to 5 and from 0.1 to 0.0001, respectively, for each AE-based model. Unless otherwise specified, the encoding layer sizes of the AE-based algorithms were proportionally decreased from the input dimensionality to the latent dimensionality given the number of layers, and the decoding layer sizes were increased in the opposite way. The layer size and learning rate determined for each data set and algorithm are provided with the source code for reproducibility.
    
\subsubsection{\textbf{RNN-based Algorithms}} We prepared the streaming variants of the popular RNN-based anomaly detection algorithms LSTM-ED\,\cite{LSTM-ED} and REBM\,\cite{REBM}\footnote{LSTM-ED and REBM: \url{https://github.com/KDD-OpenSource/DeepADoTS}} to be usable in our problem setting. LSTM-ED is based on an encoder-decoder network with long short-term memory, and REBM is based on a recurrent energy-based model. For the streaming variants of the models, referred to sLSTM-ED and sREBM, we trained the models incrementally by continuously feeding incoming batches over data streams. Epoch number, batch (and mini-batch) size, and latent space size were set in the same way as applied in the AE-based algorithm. The source codes of the two algorithms are available in the public repository.

\subsubsection{\textbf{Streaming Algorithms}} We used the five existing state-of-the-art streaming anomaly detection algorithms: MStream\,\cite{MSTREAM}, STARE\,\cite{Yoon20}\footnote{STARE: \url{https://github.com/kaist-dmlab/STARE}}, DILOF\,\cite{Na18}, MiLOF\,\cite{Sal16}, and RRCF\,\cite{Sud16}\footnote{RRCF: \url{https://github.com/kLabUM/rrcf}}.  STARE employs kernel density estimation\,(KDE)-based local outlier detection, MiLOF and DILOF employ kNN for local outlier factor\,(LOF), RRCF employs ensemble decision trees based on the isolation forest\,\cite{IF}, and MStream\,\cite{MSTREAM}\footnote{MStream: \url{https://github.com/Stream-AD/MStream}} employs locality-sensitive hashing with dimensionality reduction. We implemented the five algorithms based on the source codes provided by the authors or publicly available. Similarly to the AE- or RNN-based algorithms, incoming batches of 512 data points over data streams were fed to each algorithm for continuous inference and update.

While most hyperparameters in these algorithms are set to the default values suggested by the authors, \cmmnt{most} sensitive hyperparameters \cmmnt{in each algorithm} are tuned to achieve the best accuracy. For STARE, MiLOF, and DILOF, the number of neighbors (for KDE-based local outlier scores or LOF scores) was tuned between 2 and the batch size. For RRCF, the number of trees was tuned from 2 to 16, and the size of the trees was tuned between the batch size and ten times the batch size. For MStream, the temporal decaying factor was tuned between 0 and 1 while using the AE-based dimensionality reduction which showed the best results in the original paper.


\subsection{Performance Comparison of the Streaming Variants of AE-based Algorithms}
\label{apdx:sAE_results}
Table \ref{tbl:sAE_performance} shows the AUC results of the streaming variants of AE-based algorithms (i.e., sRAPP, sRSRAE, and sDAGMM) which are omitted in Table \ref{tbl:overall_performance} in the main paper. The comparison between these streaming variants and their corresponding \algname{} instances---sRAPP and \algname{}$_{RAPP}$, sRSRAE and \algname{}$_{RSRAE}$, and sDAGMM and \algname{}$_{DAGMM}$---directly shows the benefit of our adaptive model pooling technique.
The streaming variants showed much lower accuracy than the corresponding \algname{} instances, because the variants have limitations in handling diverse concept drifts.

\input{Figures/OverallPerformance_sAE.tex}

\subsection{Scalability of Compared Algorithms}
\label{apdx:scalability}
Figure \ref{fig:scalability_others} shows the scalability evaluation results of \algname{}$_{RAPP}$ for varying batch size in the seven data sets omitted in Figure \ref{fig:scalability} in the main paper. Note that both x-axis and y-axis are in a log scale. Similarly, the time \algname{} took was less than a second for processing hundreds of high-dimensional data points, and its rate of increase was comparable to that of the baseline for all data sets with different concept drifts. This result, again, demonstrates that \algname{} is scalable with respect to varying input data rates regardless of concept drift types.

%% file: Figures/OverallPerformance_sAE.tex
\begin{table}[t]
\center
\small
\caption{The AUC results of the streaming variants of AE-based algorithms omitted in Table \ref{tbl:overall_performance}.}
\vspace{-0.2cm}
\label{tbl:sAE_performance}
\begin{tabular}[c]{@{}C{1.2cm}C{2.2cm}|C{1.0cm}C{1.0cm}C{1.2cm}@{}}
\toprule
\multicolumn{1}{l}{}  &  Data set    & sRAPP & sRSRAE   & sDAGMM \\ \toprule
\multirow{7}{*}{{Synthetic}}                                                         & \begin{tabular}[c]{@{}c@{}}MNIST-AbrRec\end{tabular}         & \begin{tabular}[c]{@{}c@{}}0.860\\ ($\pm0.006$)\end{tabular} & \begin{tabular}[c]{@{}c@{}}0.756\\ ($\pm0.006$)\end{tabular} & \begin{tabular}[c]{@{}c@{}}0.637\\ ($\pm0.010$)\end{tabular} 
\\ \cdashline{2-5}
& \begin{tabular}[c]{@{}c@{}}FMNIST-AbrRec\end{tabular}       & \begin{tabular}[c]{@{}c@{}}0.697\\ ($\pm0.018$)\end{tabular} & \begin{tabular}[c]{@{}c@{}}0.706\\ ($\pm0.006$)\end{tabular} & \begin{tabular}[c]{@{}c@{}}0.640\\ ($\pm0.012$)\end{tabular} \\  \cdashline{2-5}

& \begin{tabular}[c]{@{}c@{}}MNIST-GrdRec\end{tabular}       & \begin{tabular}[c]{@{}c@{}}0.813\\ ($\pm0.019$)\end{tabular} & \begin{tabular}[c]{@{}c@{}}0.752\\ ($\pm0.004$)\end{tabular} & \begin{tabular}[c]{@{}c@{}}0.633\\ ($\pm0.017$)\end{tabular} \\ \cdashline{2-5}

& \begin{tabular}[c]{@{}c@{}}FMNIST-GrdRec\end{tabular}       & \begin{tabular}[c]{@{}c@{}}0.718\\ ($\pm0.020$)\end{tabular} & \begin{tabular}[c]{@{}c@{}}0.665\\ ($\pm0.004$)\end{tabular} & \begin{tabular}[c]{@{}c@{}}0.610\\ ($\pm0.006$)\end{tabular}\\ \midrule

\multirow{7}{*}{\begin{tabular}[c]{@{}c@{}} Real\\ (known \\drifts)\end{tabular}}   & \begin{tabular}[c]{@{}c@{}}INSECTS-Abr\end{tabular}    & \begin{tabular}[c]{@{}c@{}}0.601\\ ($\pm0.017$)\end{tabular} & \begin{tabular}[c]{@{}c@{}}0.797\\ ($\pm0.002$)\end{tabular} & \begin{tabular}[c]{@{}c@{}}0.596\\ ($\pm0.009$)\end{tabular} \\ \cdashline{2-5}


   & \begin{tabular}[c]{@{}c@{}}INSECTS-Inc\end{tabular}    & \begin{tabular}[c]{@{}c@{}}0.528\\ ($\pm0.026$)\end{tabular} & \begin{tabular}[c]{@{}c@{}}0.765\\ ($\pm0.002$)\end{tabular} & \begin{tabular}[c]{@{}c@{}}0.545\\ ($\pm0.024$)\end{tabular}   \\ \cdashline{2-5}
   
   & \begin{tabular}[c]{@{}c@{}}INSECTS-IncGrd\end{tabular} & \begin{tabular}[c]{@{}c@{}}0.553\\ ($\pm0.010$)\end{tabular} & \begin{tabular}[c]{@{}c@{}}0.827\\ ($\pm0.001$)\end{tabular}  & \begin{tabular}[c]{@{}c@{}}0.628\\ ($\pm0.025$)\end{tabular}   \\ \cdashline{2-5}

   & \begin{tabular}[c]{@{}c@{}}INSECTS-IncRec\end{tabular} & \begin{tabular}[c]{@{}c@{}}0.562\\ ($\pm0.024$)\end{tabular} & \begin{tabular}[c]{@{}c@{}}0.786\\ ($\pm0.001$)\end{tabular}  & \begin{tabular}[c]{@{}c@{}}0.606\\ ($\pm0.011$)\end{tabular} \\ \midrule


\multirow{1}{*}{\vspace*{1.0cm}\begin{tabular}[c]{@{}c@{}} Real\\(unknown \\ drift) \end{tabular}} & GAS     & \begin{tabular}[c]{@{}c@{}}0.813\\ ($\pm0.011$)\end{tabular} & \begin{tabular}[c]{@{}c@{}}0.514\\ ($\pm0.007$)\end{tabular} & \begin{tabular}[c]{@{}c@{}}0.435\\ ($\pm0.031$)\end{tabular} \\ \cdashline{2-5}

   & RIALTO        & \begin{tabular}[c]{@{}c@{}}0.712\\ ($\pm0.033$)\end{tabular} & \begin{tabular}[c]{@{}c@{}}0.579\\ ($\pm0.006$)\end{tabular} & \begin{tabular}[c]{@{}c@{}}0.539\\ ($\pm0.023$)\end{tabular} \\ \bottomrule
\end{tabular}
\vspace{+0.3cm}
\end{table}